\documentclass[journal]{IEEEtran}
\usepackage{times}
\usepackage{amsmath,epsfig}
\usepackage{amssymb}
\usepackage{amsfonts}
\usepackage{array,dcolumn}
\usepackage[ruled,vlined,linesnumbered]{algorithm2e}
\usepackage{subcaption}
\usepackage{psfrag}
\usepackage{amsmath}
\usepackage{amsfonts}
\usepackage{amssymb}
\usepackage{balance}
\usepackage{accents}
\usepackage{booktabs, adjustbox}
\usepackage{cite}
\usepackage[all]{xy}
\usepackage{dsfont}
\usepackage{hyperref}
\usepackage{url}
\usepackage{xcolor}
\usepackage[nolist]{acronym}
\usepackage{graphicx}
\usepackage{float}
\usepackage{threeparttable}
\usepackage{cancel}
\usepackage[official]{eurosym}
\usepackage{chemformula}
\usepackage{subcaption}
\usepackage{graphicx}
\usepackage[utf8]{inputenc}
\usepackage[english]{babel}

\usepackage{filecontents, pgffor}
\usepackage{listings, xcolor}

\usepackage{xcolor}
\usepackage[linesnumbered,ruled,vlined]{algorithm2e}
\usepackage{mathtools,amssymb,lipsum, nccmath}
\usepackage{cuted}
\usepackage[normalem]{ulem}

\usepackage{blindtext}

\SetKwInput{KwInput}{Input}
\SetKwInput{KwOutput}{Output}

\usepackage{amsthm}
\newtheorem{lemma}{Lemma}

\newcommand{\pp}[1]{\textcolor{black}{#1}}
\newcommand{\PP}[1]{\textcolor{black}{#1}}
\newcommand{\arne}[1]{\textcolor{black}{#1}}
\newcommand{\lam}[1]{\textcolor{black}{#1}}
\newcommand{\srp}[1]{\textcolor{black}{#1}}

\newcommand{\bs}[1]{\textcolor{black}{#1}}

\newcommand{\srpnew}[1]{\textcolor{black}{#1}}
\newcommand{\LD}[1]{\textcolor{black}{#1}}

\def\undb#1{\mbox{\bf{#1}}}

\def\BibTeX{{\rm B\kern-.05em{\sc i\kern-.025em b}\kern-.08em
    T\kern-.1667em\lower.7ex\hbox{E}\kern-.125emX}}

\SetCommentSty{mycommfont}

\DeclareMathOperator*{\argminA}{arg\,min}

\hypersetup{
     colorlinks = true,
     linkcolor = black,
     anchorcolor = black,
     citecolor = black,
     filecolor = black,
     urlcolor = black
}

\newcommand{\rv}[1]{\textcolor{black}{#1}}

\begin{document}

\author{Lam Duc Nguyen,%
\thanks{Lam Duc Nguyen, Shashi Raj Pandey, Soret Beatriz, and Petar Popovski are with Connectivity Section, Electronic System, Aalborg University, Denmark. Email: \{ndl, srp, bsa, petarp\}@es.aau.dk.}
\textit{IEEE Graduate Member}, %
Shashi Raj Pandey, \textit{IEEE Member}, %
Soret Beatriz, \textit{IEEE Member},\\ %
Arne Br\"{o}ring\thanks{Arne Br\"{o}ring is Senior Key Expert Research Scientist with Siemens AG, Munich, Germany. Email: arne.broering@siemens.com}, %
and Petar Popovski, \textit{IEEE Fellow}}

\title{A Marketplace for Trading AI Models based on \\ Blockchain and Incentives for IoT Data}

\markboth{IEEE Journal Submission}% 
{Shell \MakeLowercase{\textit{et al.}}: Bare Demo of IEEEtran.cls for IEEE Journals}

\maketitle 

\begin{abstract}
\pp{As Machine Learning (ML) models are becoming increasingly complex, one of the central challenges is their deployment at scale, such that companies and organizations can create value through Artificial Intelligence (AI).} An emerging paradigm in ML is a federated approach where the learning model is delivered to a group of heterogeneous agents partially, allowing agents to train the model locally with their own data. \pp{However, the problem of valuation of models, as well the questions of incentives for collaborative training and trading of data/models, have received a limited treatment in the literature.} 
\srp{In this paper, a new ecosystem of ML model trading over a trusted Blockchain-based network is proposed. The buyer can acquire the model of interest from the ML market, and interested sellers spend local computations on their data to enhance that model's quality. In doing so, the proportional relation between the local data and the quality of trained models is considered, and the valuations of seller's data in training the models are estimated through the  distributed Data Shapley Value (DSV). At the same time, the trustworthiness of the entire trading process is provided by the Distributed Ledger Technology (DLT). Extensive experimental evaluation of the proposed approach shows a competitive run-time performance, with a 15\% drop in the cost of execution, and fairness in terms of incentives for the participants.}
\end{abstract}
 
\begin{IEEEkeywords}
Blockchain, Federated Learning, Model Trading, Data Valuation, Shapley Value.
\end{IEEEkeywords}

\IEEEpeerreviewmaketitle

\section{Introduction}

\rv{Personal IoT devices keep generating an enormous amount of sensing data that is expected to reach 79.4 Zettabytes (ZB) globally in 2025\cite{data794}. Several attempts to enhance and adapt business workflows have been made towards exploiting the provision of IoT data\cite{previous1,previous2}. In this regard, training machine learning models and data sharing are two popular uses of IoT data. Furthermore, emerging diverse platforms for accessing and sharing IoT data connects various distributed IoT devices/data sources, thereby facilitating suppliers to exchange their data \cite{infocom2019}.}

\lam{For example, in IoT systems for air quality monitoring and emission control, Air Quality Index (AQI) is a quantity defined to estimate the degree of severity for air pollution and CO2 emission levels. AQI quantifies the concentration of various particles in the air, such as PM2.5 or PM5.0, using state-of-the-art sensor devices \cite{bishoi2009comparative}. There are two most popular measurement methods for AQI: i) sensing-based \cite{jo2020development}, and ii) vision-based \cite{9184079}. In the sensing-based method, the IoT sensor devices are delivered around the area interest, e.g., city, urban, to collect the quality of the air and emission levels. These measurements are then forwarded to the central server for further analysis and calculation of AQI. In the vision-based method, the devices with an embedded camera, such as a camera station in the road, or individual mobile phones, can take photos of a specific area and send them to the server. The server then applies advanced image processing techniques on these images to derive the analysis report of air quality. However, both methods have problems due to (i) high energy consumption for collecting data and transmission, (ii) requirement for a large dataset for a high quality AQI estimation, (iii) the server acting as a single point of failure, and (iv) data privacy concerns under General Protection Regulation (GDPR). Several recent works have addressed issues related to (i) using efficient resource management techniques and (ii) with \pp{dense sensory networks \cite{moltchanov2015feasibility}.} However, the primary concerns about (iii) and (iv) remain a single point of failure network topology and data privacy protection. They are yet to be addressed in an efficient manner.
\rv{In addition, the regulations such as EU’s GDPR, California’s Consumer Privacy Act (CCPA), and China’s Cyber Security Law (CSL)\cite{gruschka2018privacy} limits the reckless use/collection of personal information and fosters data privacy.} Hence, a feasible joint solution to address challenges raised in (iii) and (iv) is imperative for optimized operation of the market for data exchange.}

\begin{figure} 
    \centering
     \includegraphics[width=0.8\linewidth]{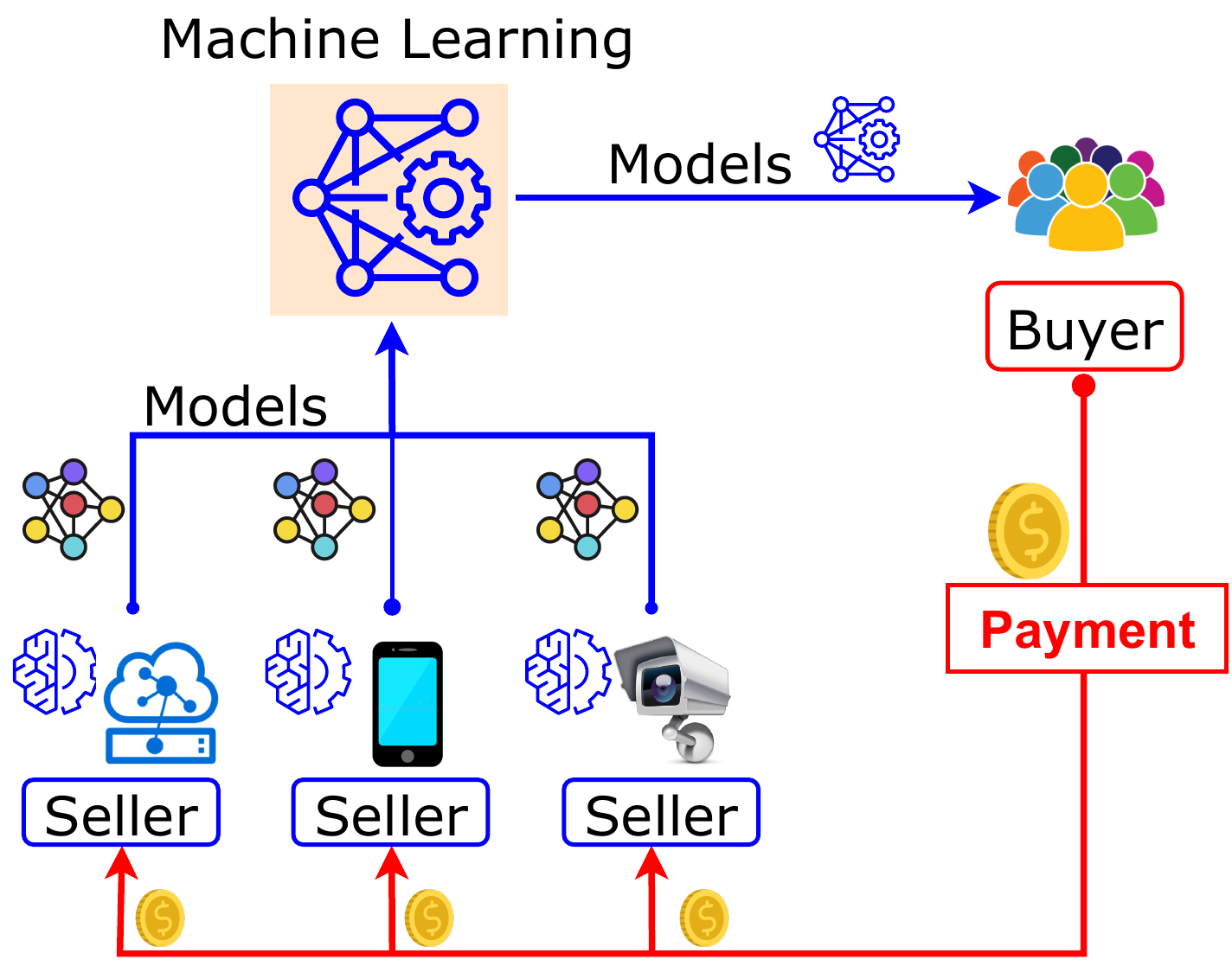}
    \caption{A motivation example: IoT devices contribute to train an ML model for the buyer to predict CO2 emission levels and get incentives from their contributions.}
    \label{fig:motivation}
\end{figure}

\lam{In this regard, more recently, Federated Learning (FL) has been considered a key solution to address the privacy issue in training learning models \cite{mcmahan2017communication}.
\rv{FL is a distributed model training paradigm that aims to solve the challenges of data governance and privacy by training algorithms collaboratively rather than transferring the data itself.}
For example, in a typical FL setting, at first, as shown in Fig. \ref{fig:motivation}, the IoT devices collect the pollution and CO2 emission levels and store them in their local database. Consider an interested organization or individual, termed a \textit{buyer}, willing to train an ML model to predict a specific area's emission level. However, they do not \pp{possess sufficiently} large datasets about sensing information or image data. In such a scenario, they can send their initial model to a model marketplace to find appropriate parties interested in contributing to the model training process. Then, the IoT devices, termed \textit{sellers}, can download the initial model and train it using their local data. After that, the IoT devices can share the updated model weights to the marketplace, where there is an aggregator to aggregate submitted local models to build global models. Based on the aggregated global models, the model buyer can use the global model to predict the emission levels \pp{with acceptable precision}.} \lam{In this manner, using the FL approach, we can address the problem of data privacy where data is locally trained without the need to transmit to a central server.} However, from a systems \arne{perspective, a shared IT environment, such as an aggregator in the marketplace, may become a single point of failure in terms of data integrity, trust, security, and transparency \cite{xu2014ubiquitous}}.

\rv{A conventional data market is often deployed as a centralized service platform that gathers and sells raw or processed data from data owners} (e.g., the trained learning models) to the consumers \cite{radhakrishnan2018streaming} \cite{niu2018achieving}. \pp{This leads to two important concerns.} 
\rv{\emph{First}, this strategy exposes the platform as a single point of security risk; the malfunctioning platform servers has serious security concerns including data leakage, inaccurate calculations results, and manipulation of data price.}
\pp{\emph{Second}, \srp{collaboration for model training} \rv{raises questions in terms of how to motivate participants to participate in such an ML training endeavor.} The incentive for each IoT client based on their contribution should be fair and transparent. These features are not present in a  standard FL setup, as in many applications there is no clear and natural incentive mechanism for involved participants to provide quality information. This calls for a carefully designed mechanisms to reward parties economically and thus incentivize participation}~\cite{langevoort1992fraud, pandey2020crowdsourcing}. For example, a fixed price per data point could motivate participants to collect massive amounts of low quality or fake data if there is no intermediary process to check quality of training data. Besides, another reason that may disincentives parties from sharing data could stem from privacy and integrity concerns regarding the use of participant’s data once it is shared. For instance, the sellers can re-use the data which has already been sold.

\rv{\pp{The aforementioned challenges can be effectively handled by a Distributed Ledger Technologies (DLTs).}\footnote{\rv{In this work, the terms \emph{Blockchain} and DLT are used interchangeably. Blockchains are a type of DLT, where nodes maintain a copy of the ledger having embedded chains of blocks. These blocks are basically composed of digital pieces of information, particularly defined as \textit{transactions}.}} \rv{DLTs and Blockchains enable untrusted parties to share information in an immutable and transparent manner\cite{bitcoin}.}}
\rv{Outside of its key role in financial transactions, the applications of DLTs can be seen as a key enabler for trusted and reliable distributed IoT systems, e.g., a distributed IoT data marketplace}
\rv{For instance, in a Blockchain-enabled IoT data marketplace \cite{nguyen2021modeling}, Blockchain transactions include IoT sensing data, or system control messages, and these are recorded and synchronized in a distributed manner in all the involved participants of the network\cite{9106844}.}
\rv{Furthermore, DLTs enable the preservation of all transactions in immutable records, with each record being spread across several participants. Thus, the decentralized nature of DLTs ensures security, as does the use of robust public-key encryption and cryptographic hashes.}
\rv{The advantages of  incorporating DLTs into trading AI models in IoT systems include: i) ensuring immutability and transparency for historical AI model trading records, ii) eliminating  the need for third parties, and iii) developing a transparent system for AI model trading in heterogeneous networks to prevent tampering and injection of fake data from the stakeholders, according to \cite{iotmagazine, 9426434}.} With the wide spread of ubiquitous marketplaces recently, \PP{it became relevant to investigate} the use of AI/ML model trading in marketplace environments.

With the aforementioned motivation, we propose a Blockchain-based model trading system which enables a secured and trusted marketplace to collaboratively train ML models as well as guarantees fair incentives for every participants and privacy of data. 
%In this paper, we propose a new distributed Blockchain-based model trading system accessible to participants to contribute their training results based on their local data to solve ML tasks.
%
%\srp{COMMENT:It appears unclear when we say ``collaborating ML models"; in the later part, it is exactly not as illustrated as in the Figure 1. We might say ``model trading system that enables a secure and trusted marketplace to collaboratively train..." . Another concern is, we are giving an impression as, "the market offers variety of ML models; the buyers offers incentives to improve the model of concern via distributed training", but later it appears little bit different.}. 
%
Based on the quality of the uploaded models, which is quantified by using a  distributed Data Shapley Value (DSV), the participants\footnote{The terms ``\textit{participants}",``\textit{clients}" and ``\textit{agents}"  in this work are used to refer to ``\textit{IoT devices}".} can get the incentive based on the updated models, for example, as tokens or fiats. Note that based on our proposed system, the parties do not need to share their local data, but only provide customized models or query interface to the marketplace. Consequently, the proposed system allows multiple participants to jointly train the ML models on the marketplace based on their own training data. Buyers who need to train their ML model will pay to the market for the improvement of their model, and sellers who sell their contribution to train the ML models will get paid by the market via smart contracts. 

The main features of the proposed model trading are:  
\begin{enumerate}
    \item \textbf{Trusted and transparent transactions}: The DLT is considered a trusted, tamper-proof, and transparent system in which the participants can check and follow the progress of a training task. Based on that, the model is exchanged and traded securely and transparently. 
    \item \textbf{Valuation of Data:} The local models contributed by the trainers (service sellers) are collected and evaluated via Shapley Value (SV) extension to approximately estimate the quality of the models. 
    \item \textbf{Fair Payment:} The participants receive their reward, which is proportional to the usefulness of their data in improving the models. The distributed incentive mechanism for FL based on SV measures participants' contributions in the marketplace.
\end{enumerate}

%The implementation  will be published on Github\footnote{https://github.com/lamnd09/dlt-AI-marketplace}. 

\subsection{Contributions and Paper Organization}
The major contributions are summarized as follows: 
%\srp{Comment: as per our earlier discussion, i think we have to clarify this now: (i) let's say 'time-relaxed' synchronous setting of FL (or distributed model training); we don't need online, tightly coupled updates (required for applications such as recommender systems); let's clearly say the use of DSV here, i.e., the contribution of each seller in improving the global model? Or the data contribution of the seller in the model training, or just the value of individual data in enhancing the local model? or other cases, (iii) first, DLT for recording the contributions (and information on trading \cite{nguyen2021modeling} to train ML models with incentives); then, the energy-efficient consensus strategy to overcome corresponding challenges following such approach.}
\begin{itemize}
    \item \textbf{ML Model Marketplace}: A Blockchain-based model trading system \srp{that allows participants to purchase learning models and sell contributions in training them. The system records the  trading details \cite{nguyen2021modeling} in a tamper-proof distributed ledger.}

    \item \textbf{Federated Data Shapley Value (SV): } \srp{We use the data SV to estimate the valuation of participants' data and
    evaluate their contribution to the model during local training. We show that the standard SV value is inefficient for distributed ML and deploy an extension of standard SV for our platform. The method is robust and allows plugging any developed mapping functions related to the device's local data into the proposed distributed Shapley mechanism for value quantification.  As a result, one can design a contribution-based, efficient incentive mechanism to stimulate model trading.} 

    \item \textbf{Performance Evaluation}: We have conducted extensive simulations and experiments, demonstrating  that the proposed approach shows a competitive run-time performance, with a 15\% drop in the cost of execution and fairness in terms of incentives for the participants.
\end{itemize}

The rest of the article is organized as follows.
Section II, presents the concepts of DLTs, FL, as well as definition of data valuation schemes used in this paper. This is followed by description of the system model of the marketplace for ML model trading and explain in detail how the system works. In Section III, the value of ML models is calculated using the \emph{Approximate Federated Shapley Value (AFS)}. 
Section IV contains description of the testbed and experimental results.
\pp{\rv{Section V discusses related works and finally, Section VI concludes the paper}.}

\section{\pp{Preliminaries}}
\subsection{Standard FL}

\srp{\rv{FL is a distributed machine learning setting in which numerous entities (clients) cooperate on training a learning model without disclosing their available raw data \cite{mcmahan2017communication}.}} Instead, clients distributively perform computations on their data and transfer obtained local learning parameters updates to the server for aggregation process. The aggregated model, i.e., the global model, is broadcast back to the clients for the next round of local computations resulting in the local learning parameters. The interaction between the server and clients to solve the learning problem continues until \pp{an acceptable} level of model accuracy is achieved \cite{kairouz2019advances}. In this manner, FL offers (i) privacy-preserving benefits in the model training approach by not requiring clients to share their local data to the server, and consequently, (ii) lower communication overhead by offering distributed model training paradigm and exchange of model parameters only. \rv{Therein, FL enables training AI/ML models at edge networks}
%
% Motivated by privacy concerns among data owners, the concept of FL is introduced in \cite{li2020federated}. FL allows users to collaboratively train a shared model while keeping personal data on their devices, thus alleviating their privacy concerns.
%\srp{COMMENT: I think it is better we fuse these two sentences in the first para.} 
%Therefore, FL can serve as an enabling technology for ML model training at edge networks. 

% For an introduction to the categorizations of different FL settings, e.g., vertical and horizontal FL\cite{yang2019federated}. 
%

\begin{figure}[t!]
    \centering
    \includegraphics[width=1.0\linewidth]{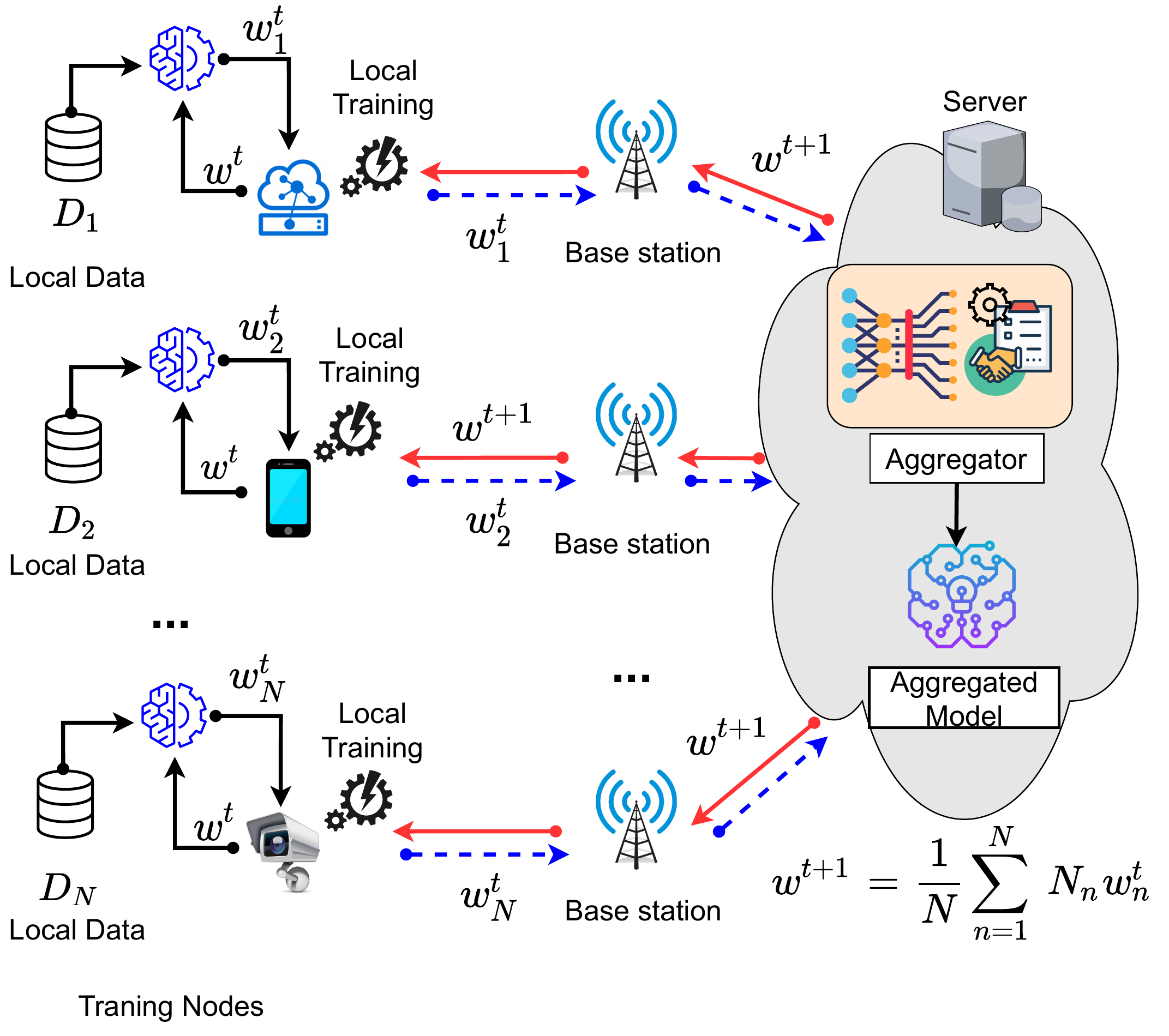}
    \caption{Standard FL.}
    \label{fig:architecture}
\end{figure}

\begin{table}[t!]
\centering
\caption{Summary of key notations.}
%\bs{not referenced in the text? missing the payment/deposit}
 \begin{tabular}{@{} | p{1cm} | p{5cm}| @{}}
\hline
 \textbf{Notation} & \textbf{Meaning}  \\ [0.5ex] 
% \hline
% DLT          & Distributed Ledger Technology \\
\hline
$ MI_i $         & DLT miner $i$     \\
\hline
% SV          & Shapley Value \\
% \hline
$\mathcal{N}$             & Set of $N$ participants (clients)   \\
\hline
$D_i$            & Private local dataset of user $i \in \mathcal{N}$   \\
\hline
 $\undb{M}_i $           & Local model of user $i \in \mathcal{N}$   \\
 \hline
$\mathcal{M}_G $           & Global model aggregated at DLTs   \\
\hline
$\mathcal{M}^0_G $           & Initial Global model at DLTs   \\
\hline
$L(\undb{M}^t_i)$ & Loss function    \\
\hline
$\eta $ & Learning rate  \\
\hline
$S_i $ & Model trainer (\bs{or Seller}) $i$    \\
\hline
$B_i $ & Model owner (\bs{or Buyer}) $i$      \\
\hline
$\mathcal{P}_d$ & Deposit from buyer \\
\hline
$\undb{B}$ & Training batch size  \\
\hline
$\mathcal{A}$    & Training algorithm     \\
\hline
$ U(\cdot) $     & Utility function   \\
\hline
$ \phi_i $       & Valuation of data contributor $i$ \\
\hline
$ E $            & Number of local epochs for a FL setting    \\
\hline
$ T $            & Number of training interactions    \\
\hline
$\mathcal{T}_i $ & A trade deal between $S_i$ and $B_i$   \\
\hline
$\widetilde{\mathcal{M}}$ & Approximated model $\widetilde{\mathcal{M}}$   \\
%$ $              &   \\
%$ $              &    \\
%$ $              &    \\
\hline
\end{tabular}
\label{tab:symbols}
\end{table}

\rv{Fundamentally, there exists two main actors in the FL system: (i) the data owners, often termed as participants, and (ii) the model owner, which is the FL server. Consider a set of $\srp{N}$ data owners, defined as $\mathcal{N}=\{1,2,\ldots,N\}$, where each of them has a private dataset $D_{i \in \mathcal{N}}$. In Table~\ref{tab:symbols}, we provide the summary of key notations used in this paper. Each data owner $\srp{i}$ trains a local model $\undb{M}_i$ using its dataset $D_i$ and sends only the obtained local model parameters to the FL server. Then, the FL server aggregates all the collected local models to build a global model, $\mathcal{M}_G = \sum_{i \in \mathcal{N}} M_i $. This is where, in principle, the FL approach differs from the traditional centralized training where $D = \cup_{i\in \mathcal{N}} D_i$ is used to train a model $\mathcal{M}_T$, i.e., data first gets aggregated centrally before the actual model training happens. In Fig. \ref{fig:architecture}, we show a standard architecture and an overview mechanism of the FL training process.
% In this system, the data owners serve as the FL participants which collaboratively train an ML model required by an aggregate server. 
We assume that the data owners are honest, i.e., actual private data will be used for the local training, and correspondingly, the FL server will receive accurate local models from the data owners. Following to which, the workflow of standard FL can be described as below.}

\rv{First, considering the target application, the server decides the training task and defines the corresponding data requirements. Furthermore, the server also specifies the hyper-parameters of the global model and the training process, e.g., the learning rate \srp{$\eta$}. The server then broadcasts the initialized global model $\mathcal{M}^0_G$ and the learning task to a subset of selected participants.}
Next, based on the global model $\mathcal{M}^t_G$, where $t$ denotes the current global iteration index, i.e., the communication rounds between the participants and the server, each participant uses its local data to update their model parameters $M^t_i$. \rv{In doing so, during iteration $t$, the participant \srpnew{$i \in \mathcal{N}$} aims at finding the optimal parameters $\mathcal{M}^t_i$ that minimize the local loss problem $L(M^t_i)$, defined as the finite-sum of empirical risk functions as} 

\begin{equation}
    \mathcal{M}^t_i = \argminA_{M^t_i} L(M^t_i).
\end{equation}
\rv{We formally call it \textit{local iteration}. Subsequently, the obtained local model parameters from participants are sent back to the server, where they are aggregated to get the global model parameters $M^{t+1}_G$. Eventually, the global model is then broadcast back to the data owners for the next round of local iteration; hence, the iterative process is continued. In doing so, the server minimizes the global loss function $L(M^t_G)$ as the following approximation in the distributed setting of FL:}

\begin{algorithm}[t!]
\DontPrintSemicolon
  
   \KwInput{Local mini-batch size $\undb{B}$, number of participants per interaction $m$, number of global interactions $T$,  number of local epochs $E$, and learning rate $\eta$}. 
  \KwOutput{ Global Model $\mathcal{M}_G$}.
  \textbf{LocalTraining}($i$, $\undb{M}$): 
  Split local dataset $D_i$ to mini-batches of size $\undb{B}$ which are included into the set $\mathcal{B}$.\;
  \For{\normalfont{each local epoch $j$ from $1$ to $E$}} 
  {
    \For{each $b \in \srp{\mathcal{B}}$}{$\undb{M} \leftarrow \undb{M} - \srp{\eta \nabla L(\undb{M}; b)}$;\; } 
%    \tcc{$\Delta L$ is the gradient of $L$ on $b$}
  }
  \textbf{Return} $\undb{M}$ to the server.\;
  
%  \tcc{Server Side - Global Aggregation}
  
  \textbf{Initialized} $\mathcal{M}_G^0$.\;
  \For{\normalfont{each interaction} $t=\{0, 1,2,3,\ldots,T-1\}$}
  {
%        \tcc{Randomly choose $\mathcal{S}_t$ of $m$ \in $\mathcal{N}$ }
        
        $\srp{\mathcal{S}_t} \leftarrow $  (\normalfont{random set of $m$ clients});\;
        \For{ \normalfont{each participant} $i \in \mathcal{S}_t$ \normalfont{\textbf{in parallel}} }
        {
            $\srp{\undb{M}_i}^{t+1} \leftarrow \textbf{LocalTraining}(i, \mathcal{M}_G^t)$;\;
        }
        $\mathcal{M}_G^{t+1} = \frac{1}{\sum_{i\in \mathcal{N}} D_i } \sum_{i=1}^N D_i \undb{M}_i^{t+1}$;\; 
        
%        \tcc{Averaging aggregation}
  } 
\caption{Federated Averaging (FedAvg) Algorithm}
\end{algorithm}

\begin{equation} 
L\left ({\mathcal{M}_{G}^{t}}\right)=\frac {1}{N}\sum \nolimits_{i=1}^{N}L\left ({\mathcal{M}_{i}^{t}}\right). \tag{4}
\end{equation}
%\srp{COMMENT: In Alg. 1 and other parts describing the FL process, the notations appear not consistent. I think we can improve this. For instance, we can just use local model parameters for data owner $i$ as $M_i$, if not the standard notation $w_i$ }
Note that the \srp{FL process} can train different ML models that essentially use the SGD method such as Support Vector Machines (SVMs), neural networks, and linear regression. 

\rv{However, a single server dependency in the traditional FL framework makes the system vulnerable to threats, such as when the server behaves maliciously. Therefore, integrating FL with DLTs should be a promising approach to address limitations\cite{kim2019blockchained}}. 

\begin{figure}[b!]
    \centering
    \includegraphics[width=1.0\linewidth]{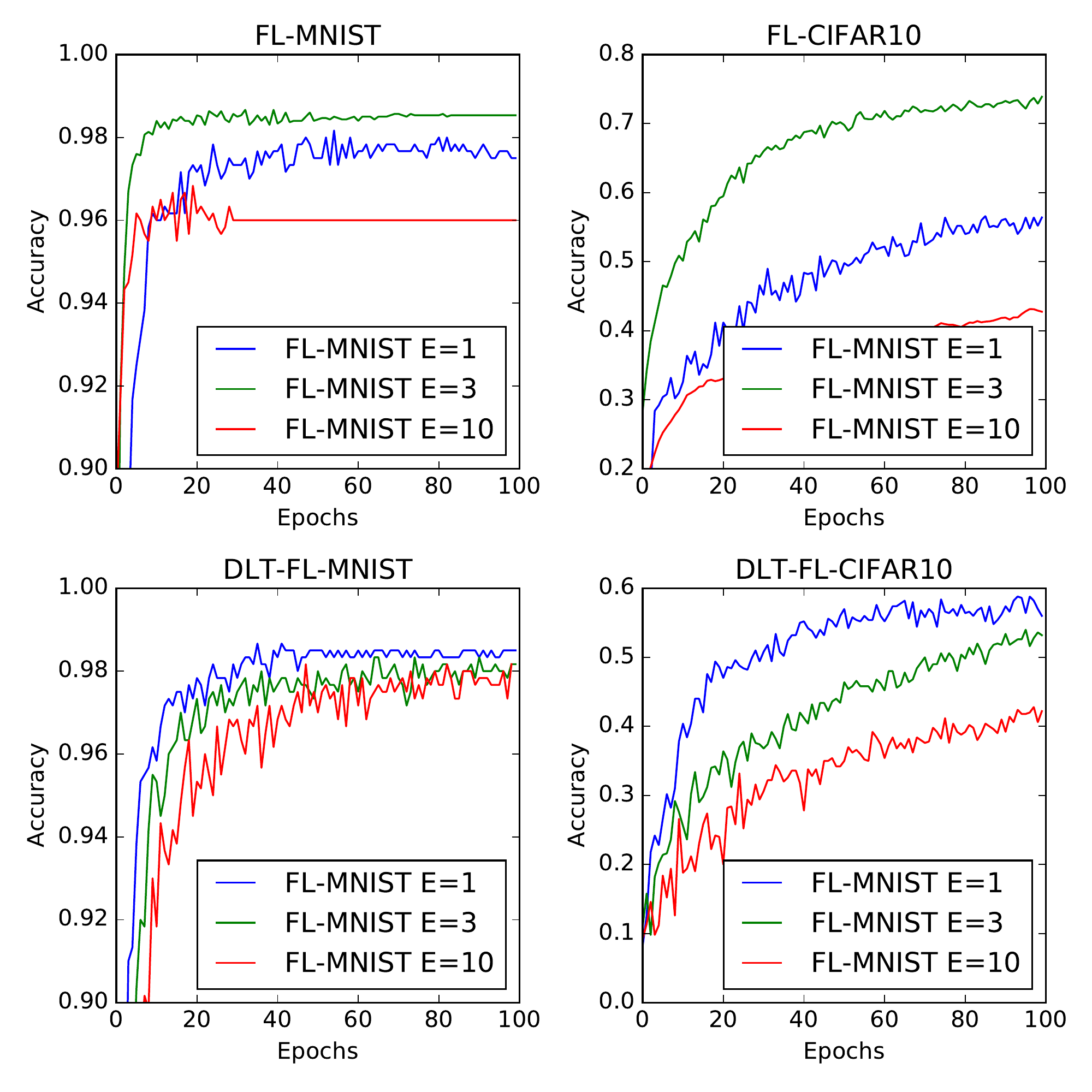}
    \caption{The accuracy of Standard FL and DLT-based FL.
    \label{fig:standardcomparison}}
\end{figure}

\subsection{Distributed Ledger as a Service for FL}
\rv{DLT is a peer-to-peer distributed ledger that records transactions in a network in a transparent and immutable manner. Besides, smart contract, which is considered as a key innovation in DLT/Blockchain area, provide programmability contracts to the DLTs, in the sense that the defined agreements in contracts are executed autonomously.}
%The DLT is a peer-to-peer ledger that records transactions in a network \cite{bitcoin}. Blockchains are a type of DLT, where chains of blocks are made up of digital pieces of information called transactions and every node maintains a copy of the ledger. In DLT, transactions are recorded and synchronized in a distributed manner in all the participants of the system. These participants are called miners or peers and, in some specific DLTs, users are charged a transaction fee to perform (crypto) transactions\cite{9429984}.
%
%Besides, smart contracts bring programmability contracts to the DLTs, in the sense that Smart Contracts execute defined rules autonomously. Smart contracts are deployed   in the DLTs with   specific addresses,  so in order to invoke a function written in a smart contract, transactions are signed off by nodes and addressed to  the  smart  contracts  themselves. While  smart  contracts enforce  terms  and  conditions  for  transactions  in financial applications, they can enforce access control policies in the modular consortium architecture. 
With the mentioned nature advantages of DLTs and smart contract, the FL framework running on the top of DLT should be completely distributed and avoid the single point of failure issue. 

\srp{In the DLT-based FL,} we assume each client device is always connected to one of the DLT miners and, if the physical connection with the current DLT miner becomes unavailable, then the device will be automatically associated with another DLT miner.
In each miner-device pair, the DLT miner works as the leader of the associated IoT devices, and they are responsible for uploading and downloading data or training models.
\rv{During the training process, the IoT device downloads the latest global model recorded in the ledger and trains \srp{ for the updated version of the} local model using their private local data. After completing the local training, the device uploads the local model to the paired DLT miner and \srp{the global aggregation process starts}. In the training time, all involved IoT devices are allow to download the latest information of associated DLT miners to receive the evaluation of the IoT devices and global model updates. 
Finally, each IoT device publishes its local training model and enters to a new round of local iteration using the newest version of the obtained global model. In this manner, the iterative ML model training process is operated until the global model has achieved a satisfactory accuracy or convergence.}

Each miner has its verifier and block to ensure that the real models and the contributions of devices are updated. Each block contains a head and body parts. The blockhead contains a pointer to the next block, and the body part contains a set of validated transaction information. The local models are formed in transaction format and in order to make the solution scalable, the local models are recorded in IFPS storage, such that just a hash version of the models is recorded in the distributed ledger.
\LD{The basic comparison between standard FL and DLT-based FL is presented in Fig.\ref{fig:standardcomparison}. The accuracy is similar in both standard FL and DLT-based FL, but the time required for convergence of DLT-based FL is higher than standard FL because of extra verification and consensus in the system.}

\subsection{Data valuation using Shapley value}

\rv{Game theory is an economic tool best-suited to analyze a system where two or more participants get involved in to achieve a desired payoff. The Shapley Value (SV) is a solution concept of fairly distributing the incentive and payoff for the involved parties in coalition \cite{roth1988shapley}. In this regard, the SV applies mainly in scenarios where the contributions of each involved participant are unequal, but all the participants work in cooperation with each other to achieve the payoff. 
The SV of user $i$ is defined as the average marginal contribution of $i$ to all possible subsets of $D=\{D_1, D_2,\ldots, D_N\}$ formed by other users as}

\begin{equation}
    \phi_i(N, U) = \frac{1}{N!} \sum_{\mathcal{S} \subseteq \mathcal{N} \setminus \{i\}} \frac{U(M_{\mathcal{S}\cup \{i\}}) - U(M_\mathcal{S})}{\binom{N-1}{|\mathcal{S}|}},
    \label{eq:standard shapely}
\end{equation}
where the function $U(\cdot)$ gives the value for any subset of those users, e.g., let $\mathcal{S}$ be a subset of $\mathcal{N}$, then $U(\mathcal{S})$ gives the value of that subset. 
This captures the average \srp{value of the} contributions of user $i$ for subsets of all coalition of users.
% averaging over all the different sequences according to which the grand coalition could be built up from the empty coalition.
%
Intuitively, assume that the user's data is to be collected in a random order, \srp{and that every user $i$ receives its marginal contribution for the collected data.} If we average these contributions over all the possible orders of $N$ users, we obtain $\phi_i(N, U)$. The importance of the SV is that it is the unique value division scheme that satisfies the following desirable properties described as follows.
\begin{itemize}
    \item \textit{\textbf{Symmetry}}: For all \srp{$\mathcal{S}\subseteq \mathcal{N}\setminus \{i, j\}$}, if user $i$ and $j$ are interchangeable, and $U(\mathcal{S} \cup \{i\}) = U(\mathcal{S} \cup \{j\})$, then, $\phi_i = \phi_j$. Thus, the users $i$ and $j$ contribute the same amount to every coalition of the other agents. Besides, the symmetry axiom states that such agents should receive the same payments.
    
    \item \textit{\textbf{Dummy User}:} User $i$ is considered as dummy user if the amount that $i$ contributes to  coalition is exactly the amount that $i$ is able to achieve alone, i.e., $\forall \mathcal{S}, i\notin \mathcal{S}, U(\mathcal{S} \cup \{i\}) - U(\mathcal{S}) = U(\{ i\} )$. \rv{According to the dummy user axiom, dummy users should be compensated exactly for the amount they achieve on their own. Users that make zero marginal contributions to all subsets of the data set, on the other hand, earn no compensation,} for example, $\phi_i =0$ if $U(\mathcal{S} \cup \{i\})=0, \forall  \mathcal{S} \subseteq \mathcal{N} \setminus \{i\}$. 

    \item \textit{\textbf{Additivity}:} For any two $U_1$ and $U_2$, we have for any user $i$, $\phi_i(N,U_1 + U_2) = \phi_i(N,U_1) + \phi_i(N,U_2)$, where the game $(N,U_1 + U_2)$ is defined by $(U_1 + U_2)(\mathcal{S}) = U_1(\mathcal{S}) + U_2(\mathcal{S})$ for every coalition $\mathcal{S}$.
\end{itemize}

Based on these background knowledge, we designed a distributed marketplace for trading AI/ML models based on Blockchain and incentive mechanism in IoT environment. 

\begin{figure*}
    \centering
    \includegraphics[width=0.8\linewidth]{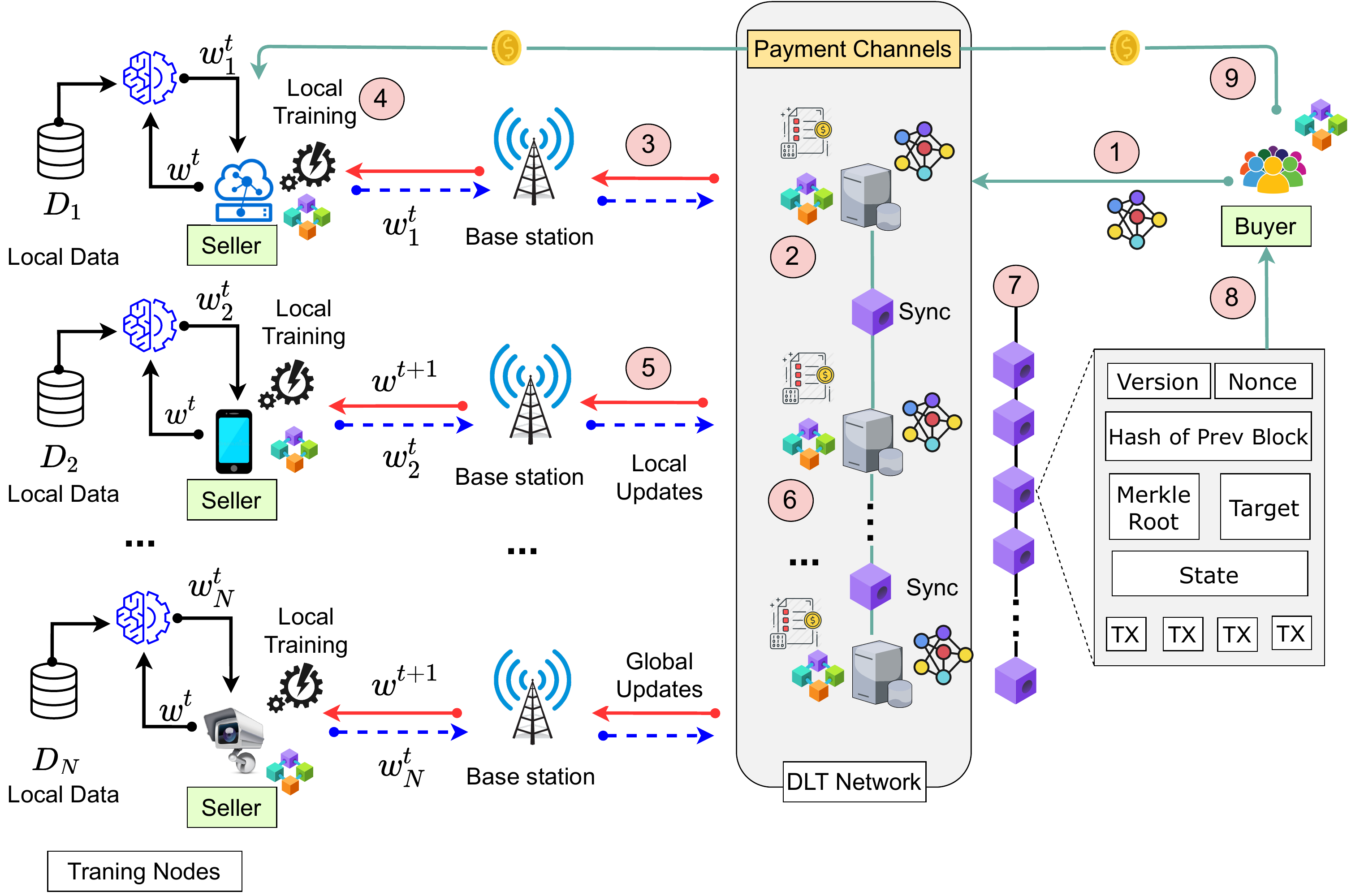}
    \caption{DLT-based ML Model Trading Framework Architecture. }
    \label{fig:system}
\end{figure*}

\section{System Design and Analysis}
\subsection{System Components}

The general architecture of DLT-based model trading includes three main components: model owners or buyers  $B$, model trainers or sellers $S$, and a distributed ledger, shown in Fig.~\ref{fig:system}. \srpnew{We assume that each seller or buyer owns one device in the network}. 
%\bs{COMMENT: I would simplify and assume 1 seller = 1 device, and then the seller/buyers are not necessarily all humans like in the healthcare example but it can also be automatic processes (selling satellite data, for example)} Here we assume that buyers and sellers act as digital wallets in a distributed network.
\PP{Within a deal (a trade) by $\mathcal{T}_i$,} the seller $S_i\in S$ and buyer $B_i\in B$ communicate using wireless links. The AI/ML model trading procedure occurs to complete a trade between $S_i$ and $B_i$,  exchanging models $\mathcal{M}_i\in\mathcal{M}$ and payment $\mathcal{P}_i$. First and foremost, $B_i$ completes the deposit $\mathcal{P}_d$ to \PP{$S_i$} via smart contracts in reference to the requested training models, $\mathcal{M}_i$. After the sellers complete the requests of the buyers, in terms of accuracy, convergence time, etc, the smart contracts are autonomous executed to pay for the effort of sellers using the amount of deposit $\mathcal{P}_d$ from buyers. \srpnew{Following Fig.~\ref{fig:system}, the general procedure of interaction between a single buyer $B_i$ and a single seller $S_i$ can be described as follows}:

%\bs{the explanation of the communication infrastructure is missing, there are "base stations" in the figure and then in the DLT part you talk about "broadcasting" --> explain the wireless + wired parts}

\subsubsection{Model Owner $i$ (as buyer $B_i$)} $B_i$ could be an individual or organization who needs model training. $B_i$ sends a request $b_i$ including task type, budget, deposit, amount of data, quality of data, price, discount, etc, to the marketplace via smart contracts. 
$b_i$ will be transmitted to selected $S_i$ and recorded in the ledger via transaction $T_{i, \textrm{add}}$. After receiving the trained and aggregated models from $S_i$ and marketplace which fulfills requirements regarding to, e.g., accuracy, the $B_i$ generates a transaction $T_{i, \textrm{commit}}$ which executes payment from $B_i$'s wallet to smart contract, forwards the payment to sellers, and generates a acknowledgment message back to the distributed ledger.

\subsubsection{Model Trainer (as seller $S_i$)} 
\rv{The model trainers play two main roles in the system: i) collects sensing data from the environment (e.g., data from surveillance systems, environmental sensing data, and  geographical data),  or acts as a data hub gathering data from nearby physical devices; ii) subscribes the model training requests from the buyers, and train the models downloaded from the marketplace with the local data.}
\rv{Seller $S_i$ earns the payment $\mathcal{P}_i$ from $B_i$ after successful delivery of $M_i$ to $B_i$. After the trained models achieve a certain accuracy based on the predefined agreements in the smart contract system, upon the appearance of $T_{i,\textrm{commit}}$ generated by $B_i$, the seller $S_i$ can receive the payment, e.g., via tokens, $\mathcal{P}_i$, \srpnew{which is in fact the deposited amount $\mathcal{P}_d$ by the buyer $B_i$,} from the marketplace via smart contract. Finally, it confirms to the distributed ledger that the deal $\mathcal{T}_i$ is completed via an acknowledgment message.}

\subsubsection{Distributed Ledgers}

\rv{The Blockchain maintains a distributed ledger that stores the history of all traded models in the form of blocks, which are connected in a chronological order. On top of that, the smart contracts are deployed to autonomously control the order and execute payments, e.g,. large payment or micro-payments from involved participants without the need of human intervention. In a distributed manner, the smart contracts ensure transparency, trust and automotive of exchanging data among parties.} These features can be deployed based on the negotiation between model owners and customers via $T_{i, \textrm{deploy}}$. \rv{Furthermore, any change in smart contracts, for example, the amount of data or the model price, or updates in the discount offers, can be made via $T_{i, \textrm{update}}$.}

\rv{In a trading system, there are an enormous amount of data exchanged among parties. Thus, increasing the number of transactions leads to slower transaction processing time and, consequently, the system's overall speed. This is reasonable as every Blockchain node needs to store and execute a computational task to validate every single transaction. Therefore, to minimize the cost of storage and execution, the trading system should record only the important data, such as payment history, aggregated global models, which could be hashed and recorded at the distributed ledger. Meanwhile, the raw data can be recorded in the distributed off-line storage component. In detail, after both model sellers $S_i$ and model owners $B_i$ have fulfilled requirements defined by smart contracts, the $T_{i, \textrm{settle}}$ is autonomously executed to query the payment $\mathcal{P}_i$ from $B_i$. Then, the payment $\mathcal{P}_i$ is transferred to $S_i$'s private wallet, while the aggregated model $\mathcal{M}_i$ is delivered to the storage address of $B_i$. In the scope of this study, we assume that the data services (e.g., data storage, trading and task dispatching) are implemented on top of a permissionless Ethereum Blockchain \cite{ethereum}. In this work, the control data and AI/ML models are formatted into normal Ethereum transactions. Furthermore, in order to improve efficiency, only the digest of each transaction is recorded in the distribute on-chain ledger, and the raw data is stored off-chain by using IPFS (InterPlanetary File System).}

\subsection{Communication Workflow}
In the DLT-based FL model trading network, we \srp{revisit the notation used to denote data owners and define the number of participants} as $\mathcal{N}=\{1,2,3,\ldots, \srp{N}\}$. The miner $MI_i$ of DLT network is associated randomly with the IoT device. 
\srp{For simplicity,} we consider the case which one miner is assigned to each physical IoT device. Each IoT device has to determine its own \srp{learning} task-related dataset size and upload it to the ledger system to receive a reward. \srp{A distributed SV incorporates the quality of local data\footnote{\srp{The quality of data signifies the size of dataset used in model training, similar to \cite{jia2019towards}. In this study, we do not consider feature attributes of data to quantify its quality.}} to determine the corresponding quality of the local model. We realize that, in some case e.g healthcare, it would be better to extract features and measuring the quality based on real quality and not quantity, but it is out of this research scope.} 
\rv{In addition, in resource-constraint IoT environments, to reduce latency and optimize energy consumption, each device's local controller performs local optimizations to establish the best scheduling policy for device resources, such as CPU cycles scheduling.}
%The local speed controller of each device performs local optimizations to determine the optimum schedule of the device resources (such as the CPU) to ensure that the local device saves the maximum energy and has the lowest latency.
%\usepackage{}

The workflow of the system is described as below. 

\textit{Step 1 (Model Initialization)}: The buyer $B_i$ initiates a model $\mathcal{M}_i$ which needs to be trained and publishes to the DLT-based marketplace. The initial model is formed in the DLT transaction format $T_{\textrm{publish}}$.
    
\textit{Step 2 (Publish initialized model to the ledger)}: Then, the transaction $T_{\textrm{publish}}$ including initialized model $\mathcal{M}_i$ is verified and recorded to the distributed ledger. At the same time, there might be many available models on the marketplace. 
    
\textit{Step 3 (Model Seller download train-required models)}: \srp{The potential seller $S_i$ can see the list of available models on the marketplace and choose to download a copy of one or multiple models of interest to train with its local data of $S_i$. 
}

% choose one or multiple models to download \srp{a copy of it}to their system to train with its own local data $S_i$. 
    
\textit{Step 4 (Local Model Training)}: After downloading the models from the distributed ledger, the sellers train the model based on their local data. The device, \srpnew{i.e., the seller $S_i$,} has its own local dataset \srpnew{$D_i$}. \srp{Local model training aims} to minimize the loss function $f(\undb{M}_i, S_i)$, where $\undb{M}_i$ is the local model of device $d_i$ and $D_i$ is its local dataset. 
 %\srpnew{Pointer2: $d_i$ should be $S_i$, $f(\undb{M}_i, S_i)$ should be $f(\undb{M}_i, D_i)$. let's not use $d_i$, but just $S_i$. We are not exactly using $d_i$}
    
\textit{Step 5 (Local trained model is updated to the ledger)}: Next, the device \srpnew{$S_i$} is randomly associated with the miner $MI_i$ \srp{to which it uploads the trained local model to the distributed ledger via smart contract.} \lam{The smart contract has functions to record the updated local models from clients via DLT interface, e.g., Web3.} 
%\bs{does it potentially change in each iteration or only in each full trade? we should make it more clear in section II.D} 
    
\textit{Step 6 (Cross-verification of the local models):} 
\rv{After receiving the local model published by IoT device in the format of transactions, the DLT miner $MI_i$ put the local model in newly generated blocks and broadcasts the model to other DLT miners in the network. Next, until other DLT miners receive the broadcasted blocks, including the local models of clients, they will verify the accuracy of local models and put the models to the new generated blocks. During this process, all the aggregated models are broadcasted to the all DLT miners in the system, and DLT miners will compare the consistency and accuracy among aggregated models. To that end, the most one will be chosen as the correct global model. Then, DLTs miners record the correct global model and the contribution of the IoT devices into the distributed ledger via smart contracts features. Otherwise, the rest of global models are considered as faulty updates.}

\textit{Step 7 (The generation and propagation of blocks)}: \rv{In order to generate a new block in the distributed ledger, DLT miners need to compute a block hash for mining and solve a cryptographic puzzle based on SHA-256, \srp{which is a one-way hash function}. As defined in popular Proof-of-Work (PoW) Blockchain, e.g., Bitcoin, Ethereum, DLT miners perform a PoW algorithm until it finds a desired nonce value or receives a new generated block from other DLT miners\cite{9271890}. There is a case, however, that the $MI_i \in MI$ acts as the DLT miner that finds the needed nonce value at the earliest, and its candidate block is generated as a new block and propagated to the other DLT miners in the network. Meanwhile, the chain can be engaged in forked problem in which multiple DLT miners find out a nonce value at the same moment. To address this issue, we use an ACK message that allows DLT miners to transmit only when each DLT miner gets the new block, which determines whether there is a fork on the main chain. Then, the DLT miner $MI_0$, which creates that newly generated block, will wait for a waiting time defined by the block ACK. Otherwise, if a fork is generated again, the process back to to previous phase to resolve the issue.}

\textit{Step 8 (Settlement)}: After the model accuracy achieves a particular value in the smart contract, the smart contract settles the deal between buyer and sellers. The finalized model is updated to the buyer and the incentive is funded to the sellers. 
    
\textit{Step 9 (Incentive to Sellers)}: Based on the contribution of each seller, the smart contract computes their contribution and transfer to sellers appropriate funds. \lam{The smart contract provides a mechanism of transparent and immutable recording and accounting contribution logs on the distributed ledger. Based on the contribution history from ledger, the clients can receive the incentives and rewards in tokens via off-chain payment channels.}

%\pp{PP: Does the smart contract have some independent data for testing and validation?}
    
\textit{Step 10 (Record receipt to the Ledger)}: All bills and receipts are recorded immutably in the distributed ledger, which allows participants to check and control their deal. \lam{Besides, we also implemented the off-chain storage solution named IPFS to store hashes of data locations on the ledger instead of raw data files. The hashes can be used to query the exact file or models through the DLT systems.}

\subsection{Distributed Shapley Value (DSV) Calculation}

\begin{algorithm}[t!]
\caption{Standard Federated Shapley FedAvg}
\DontPrintSemicolon
  \KwInput{ Local mini-batch size $\undb{B}$, number of participants $m$ per interaction, $T$ is number of global interactions,  number of local epochs $E$, and learning rate $\eta$. } 
  \KwOutput{: $\mathcal{M}^T$}.
  \textbf{Initialize} $\mathcal{M}_S^0$, where $S \subseteq \mathcal{N} = \{1,2,\ldots,n \}$.\;
  
  \For{\normalfont{each subset $\mathcal{S} \subseteq \mathcal{N}$} }{
  \For{ \normalfont{each round} $t = \{0,1,2,\ldots, T-1\}$ }{
  \For{ \normalfont{each client $i \in \mathcal{S} = \{1,2,3,\ldots,n \} $}}{
    Transmit $\mathcal{M}^t$ to $n$ clients;\;
    $\mathcal{M}_i^t \leftarrow ModelUpdate(i, \mathcal{M}^t)$;\; 
    $\delta_{\mathcal{S},i}^{(t+1)} \leftarrow \mathcal{M}_{\mathcal{S},i}^t - \mathcal{M}^t$;\;}
    
    $\mathcal{M}_\mathcal{S}^{t+1} \leftarrow \mathcal{M}_\mathcal{S}^t + \sum_{i=1}^n \frac{|D_i|}{ \sum_{i=1}^n |D_i|} \cdot \delta_{\mathcal{S},i}^{t+1}$;\;
    
    \For{ $i \in \mathcal{S} = \{1,2,3,\ldots,n \}$ } 
        {
            $\phi_i = \frac{1}{n}\sum_{\mathcal{S} \subseteq \mathcal{N} \setminus \{i\}} \frac{U(\mathcal{M}_{\mathcal{S} \cup \{i\}}^t) - U(\mathcal{M}^t_\mathcal{S}) }{ \binom{n-1}{|\mathcal{S}|}}  $;\;
        }  
  }}
  
  \textbf{Return}  $\mathcal{M}^T$, and $\phi_1,\phi_2,\ldots,\phi_n$.\\
  \textbf{ModelUpdate} ($i, \mathcal{M}$): \\
  \For{\normalfont{local epoch $e=\{1,2,3,\ldots,E\}$} }{
  \For{\normalfont{batch $b$}}{ 
    $\mathcal{M} \leftarrow \mathcal{M} - \eta\nabla L(\mathcal{M}; b) $;\;}
  \textbf{Return}  $\mathcal{M}^T$ to  central server.\;
  }
      
\end{algorithm}

%Computational cost is a challenge in adopting Shapley Value to distribute the performance of all $2^n$ models $M_S$ where $S \subseteq N $ which involves training $(2^n-1)$ additional federated models. This is computationally and communicationally expensive, in which the data contributors need to compute and send the local updates for the training of the models. Adopting an algorithm that can approximate the Shapley Value efficiently is a more suitable and practical option. The authors in \cite{ghorbani2020distributional} commented that SV value does not have to be 100$\%$ to be accurate. Note that, as the test set $T$ is finite, the function $U(M)$ approximately estimates the true performance of the trained model on the test distribution.
%

\begin{algorithm}[t!]
% \color{red}
\DontPrintSemicolon
  
  \KwInput{\textbf{Input}: Local minibatch size $\undb{B}$, number of participant $m$ per interaction, $T$ is number of global interactions,  number of local epochs $E$, and learning rate $\eta$.} 
  
  \KwOutput{\textbf{Output}: $\mathcal{M}^T$, and $\phi_1, \phi_2,\ldots, \phi_n$.}
  
  \textbf{Initialize} $\mathcal{M}^0, \widetilde{\mathcal{M}}^0_\mathcal{S} $, where $\mathcal{S} \subseteq \mathcal{N} = \{1,2,3\ldots,N \}.$\;
  
  \For{ each round $t = \{0,1,2,\ldots, T\}$ }
  {
    Transmit $\mathcal{M}^t$ to $i\in \mathcal{S}$ clients;\;
    $\mathcal{M}_i^t \leftarrow ModelUpdate(i, \mathcal{M}^t)$;\;
    $\delta_i^{t+1} \leftarrow \mathcal{M}_i^t - \mathcal{M}^t$ , $\forall i \in \mathcal{N}$;\;
    $\mathcal{M}^{t+1} \leftarrow \mathcal{M}^t + \sum_{i=1}^n \frac{|D_i|}{ \sum_{i=1}^n |D_i|} \cdot \delta_i^{t+1}$;\;
    \For{ each $\mathcal{S} \subseteq \mathcal{N}$ } 
        {           $\widetilde{\mathcal{M}}^{t+1}_{\mathcal{S}} \leftarrow \widetilde{\mathcal{M}}^t_\mathcal{S} + \sum_{i\in \mathcal{S}} \frac{|D_i|}{ \sum_{i \in S} |D_i|} \cdot \delta_i^{t+1}$;\;
        }  
  } 
  
  \textbf{Initialize} $m=0$. \\
  \While{ \normalfont{Convergence criteria not meet} }{
    $m=m+1$;\; 
        $\pi^{m}$: \normalfont{random permutation of clients with data samples to collaboratively train $\mathcal{M}^T$};\; 
        $v_0^{m} \leftarrow U(\widetilde{\mathcal{M}}^0_{\emptyset})$;\; 
        
        \For{$n \in \{1,2,\dots,|\mathcal{S}|\}$}{
            \eIf {$|U(\mathcal{M}_{|\mathcal{S}|}) - v_{n-1}^{m}| < $ PT }
                {$v_n^{m} = v_{n-1}^{m}$;}
                {
                    $\mathcal{S} \leftarrow \{ \pi^{m}[1],\pi^{m}[2],\ldots, \pi^{m}[n]\}$; 
                    $\mathcal{\mathcal{M}}_\mathcal{S}^T  \leftarrow \sum_{i\in \mathcal{S}} \frac{|D_i|}{\sum_{i \in \mathcal{S}} |D_i|} \cdot \mathcal{\widetilde{\mathcal{M}}}_i^T$; 
                    $v_n^{m} \leftarrow U(\mathcal{\mathcal{M}}^T_\mathcal{S} )$;
                }
            $\phi_{\pi^{m} [n]} \leftarrow \frac{m-1}{m} \phi_{\pi^{m-1}[n]} + \frac{1}{m} (v_n^{m} - v_{n-1}^{m})$;\;    
        }
  }
  \textbf{Return}  $\mathcal{M}^T$, and $\phi_1, \phi_2,\ldots, \phi_n$.\;
  
  \textbf{ModelUpdate} ($i, \mathcal{M}$): \\
  \For{\normalfont{local epoch $e=\{1,2,3,\ldots,E\}$} }{
  \For{\normalfont{batch $b$}}{ 
    $\mathcal{M} \leftarrow \mathcal{M} - \eta\nabla L(\mathcal{M}; b) $;\;}
  \textbf{Return}  $\undb{M}$ to ledger.\;
  }
  
%\caption{Approximated T\mathcal{M}C Federated Shapley Value (AFS)}
\caption{AFS Algorithm}
\end{algorithm}

The Standard Federated Shapley Value (SFSV) calculates the SV of data contributors based on equation \eqref{eq:standard shapely}. SFSV trains federated models based on the different subsets $S$ of contributors, and these models are evaluated on the standard test set.
However, computing the SV directly according to SFSV is time-consuming because models on all the combinations of data sets need to be trained and evaluated. By default, the SV computes the average contribution of a data source to every possible subset of other data sources. So that, evaluating the SV incurs significant communication and computation cost when the data is decentralized \cite{qu2020decentralized}. Consequently, for data SV in the FL environment, the methods in \cite{ghorbani2019data, jia2019towards} to calculate SV introduce extra training rounds on combinations of datasets from different data providers. Furthermore, the cost for extra rounds for training models could be expensive when the data volume is large. Therefore, there is a need for new strategies to evaluate the data value in FL. 

The main idea to that end is to exploit the gradients \srp{information} during the training process of the global model $\mathcal{M}$ to \srp{approximately reconstruct} the local models trained with different combinations of the client's datasets. \srp{Thereby, our approach (as described in Algorithm 3) eliminates the burden for the local models to be frequently re-trained to evaluate clients' contributions.}
In fact, the SV does not consider the order of data sources. \rv{However, in FL, it is of significant importance to take into account the order of data used for the model training so as to ensure a fair convergence. Furthermore, the updates of model are enforced to diminish over time by using, for example, a decaying learning rate\cite{charles2020outsized}. Hence, the sources used towards the end of the learning process could be less influential than those used earlier.} Therefore, to accommodate these attributes of learning properties in the decentralized model training paradigm of FL, we need to define new and efficient ways to compute SV. In this regard, based on the neutrality of FL, the SV for FL (FSV) could be computed in two different strategies.
%\textcolor{red}{Pointer for AFS or FSV?: two ways to incorporate TMC: (i) we do random order sampling of clients to approximate their contribution. In this case, all clients fully contribute to train the model, (ii) we do two loops, and quantify data samples contribution, and average it for each client at the end of the algorithm (i think we don't need this approach as it is supported by (i) }
%

The first method (called Single-Cal) reconstructs models by updating the initial global model $\mathcal{M}$ in FL with the gradients in different rounds and calculates the FSV by the performance of these reconstructed models. For example, if we want to reconstruct the model of $\mathcal{M}_{(i,j)}$ trained on the datasets of $D_i$ and $D_j$ of corresponding users, \srp{use} the gradients \srp{information} from sellers $i$ and $j$ in each round to update the initial global model $\mathcal{M}$ generated by the buyer. Then, the contribution is calculated using equation \eqref{eq:standard shapely}.

The second method (called Multi-Cal) calculates FSV in each  training round by updating the global model $\mathcal{M}$ from the previous training round with the value of gradients in the current training round. Next, the FSVs are aggregated from multiple rounds to get the final result. Therefore, there is no extra training process needed; these methods are considered efficient.
The main difference between these two strategies is that the first method approximates models through complete global iterations and only evaluates them to find SV afterwards. The second one approximates and evaluates models for every global iteration and calculates the marginal contribution for each global iteration. So that makes the second method more computationally expensive than the first one. To address this issue, we propose a new algorithm AFS based on the first approach with the use of \pp{Truncated Monte-Carlo (TMC) \cite{jia2019towards}} in Algorithm 3. \srpnew{In principle, AFS is an engineered derivative of the TMC algorithm to characterize clients' contributions with their available data samples in the collaborative training framework. In doing so, AFS evaluates the marginal contribution of each client instead of a subset of training data points, unlike the TMC method; thus, allowing clients to engage in the ML model trading with the offered incentive signals based on their data contributions.} 

Specifically, the first part of the algorithm shows the operation of the distributed ledger, lines 1--8 and lines 10--19. In line 1, the global model $\mathcal{M}^0$ and reconstructed models based on different chosen subsets $\mathcal{S} \subseteq \mathcal{N} = \{1,2,\ldots,N\}$ are initialized. Next, the distributed ledgers broadcast the global model $\mathcal{M}^t$ to $n$ selected clients in each global training round $t$ in line 3, and then receive the updates $\mathcal{M}^t_i$ from these clients in line 4 and the gradients of clients $\delta^t_i, \forall i \in \mathcal{S}$ for model aggregation are computed.
After that, the global model is updated in line 6 as 
\begin{equation}
     \mathcal{M}^{t+1} \leftarrow \mathcal{M}^t + \sum_{i=1}^n \frac{|D_i|}{ \sum_{i=1}^n |D_i|} \cdot \delta_i^{t+1}.
\end{equation}
Next, instead of updating all the local models $\mathcal{M}^t_i, \forall i \in \mathcal{S}$ in every global interactions $t= \{0,1,\ldots,T\}$, we can update only $n$ models $\widetilde{\mathcal{M}}^t_i$ and compute $\mathcal{M}^T_\mathcal{S}$  directly as a weighted average at the end, as 
$\widetilde{\mathcal{M}}^{t+1}_{i} \leftarrow \widetilde{\mathcal{M}}^t_i + \frac{|D_i}{\sum_{i=1}^n |D_i|} \delta_i^{t+1} $\ . 

We observe that evaluation of a model incurs a considerable cost in terms of time, especially if the test set is large. And with the basic idea of a single-round algorithm, we must reconstruct and evaluate $2^n$ models. \srpnew{Hence, we applied the method of TMC to decrease the computation cost and developed a tailored variant of TMC, i.e., the AFS algorithm, to address this issue. The details of the adapted TMC method is as follows.} First, we sample a random permutation of clients $\pi^{m}$ with their data samples used to train the global model \cite{ghorbani2019data}. After that, we scan from the first clients to the last client and calculate the marginal contribution of every new client's data in the training process. By repeating this process over multiple permutations, the approximation of SV is the average of all the calculated marginal contributions. The while loop is run until certain convergence criteria are met. In this work, we stop the loop when the average percentage change after a TMC iteration $m$ is less than a certain threshold. 

In line 21, the trained federated model $\mathcal{M}^T$ and the SVs are finally obtained. The local training part for the clients (lines 22--25) show how the clients use private data to train the model received from the distributed ledger. The clients use the classical gradient descent algorithm and report their updated local models $\undb{M}_{i|i= \{1,2,3,\ldots,n\}}$ to the distributed ledger.

\subsection{Performance bound on AFS algorithm}

An analytical bound on the AFS algorithm can be derived by taking the properties of TMC sampling approach into account~\cite{jia2019towards}. We consider AFS estimates the contribution of the individual client in the federated setting for a supervised learning task with probability at least (1-$\alpha$) that our estimator error is $\epsilon$. Then, we are interested in evaluating the general performance bound on AFS such that
\begin{equation}
    \textrm{Pr}(|\phi^F - \phi^S|\ge \epsilon) \le \alpha,
    \label{eq:bound}
\end{equation} 
where $\phi^S = < \phi^S_1, \phi^S_2,\ldots,\phi^S_n > $ is the vector of Shapley contributions generated by the standard SV and $\phi^F = <\phi^F_1, \phi^F_2,\ldots,\phi^F_n >$ is the approximation FSV using the proposed AFS method. Assume that we know the data distribution of clients to evaluate its marginal contribution. Then, the sampling $|\mathcal{S}|$ made during the evaluation process reflects the bound on the obtained approximation of AFS. Without loss of generality, we assume it is possible to quantify the range $v$ of the client's data marginal contribution in improving the global model. Then, we have the following Lemma \ref{lemma1}.

\begin{lemma}
Considering the TMC sampling approach \cite{jia2019towards} for evaluating the Shapley value, we have a minimum permutation $|\mathcal{S}|$ for a known range of the client's marginal contribution $v$ defining an upper bound of $\mathcal{O}\bigg(\frac{v}{|\mathcal{S}|}\bigg)$ on AFS such that $\alpha\ge2 \exp{\bigg(\frac{-2 |\mathcal{S}|\epsilon^2 }{v^2}\bigg)}$ satisfies for $0\le\epsilon,\alpha \le1$.
\label{lemma1}
\end{lemma}

\begin{proof}
The proof can be derived using Hoeffding theorem \cite{hoeffding1951combinatorial} for a known range of marginal contribution of clients. In practice, the distributed ledger can reuse the average of marginal contributions of clients, with known range $v$, to derive a sampling permutation $|\mathcal{S}|$. Then, we have 
$\textrm{Pr}\bigg(\sum\nolimits_{i \in \mathcal{S}\subseteq N}(\phi_i - \mathbb{E}(\phi_i))\ge \Delta\bigg)\le 2\exp{\bigg(\frac{-2|\mathcal{S}|\Delta^2}{v^2}\bigg)}$.
Taking the average of marginal contributions on the left-hand side of the inequality, and combining it with \eqref{eq:bound}, we get $\text{Pr}(|\phi^F - \phi^S|\ge \epsilon) \le 2 \exp{\bigg(\frac{-2 |\mathcal{S}|\epsilon^2 }{v^2}\bigg)}$. This concludes the proof.        
\end{proof}

\begin{figure}
    \centering
    \includegraphics[width=0.8\linewidth]{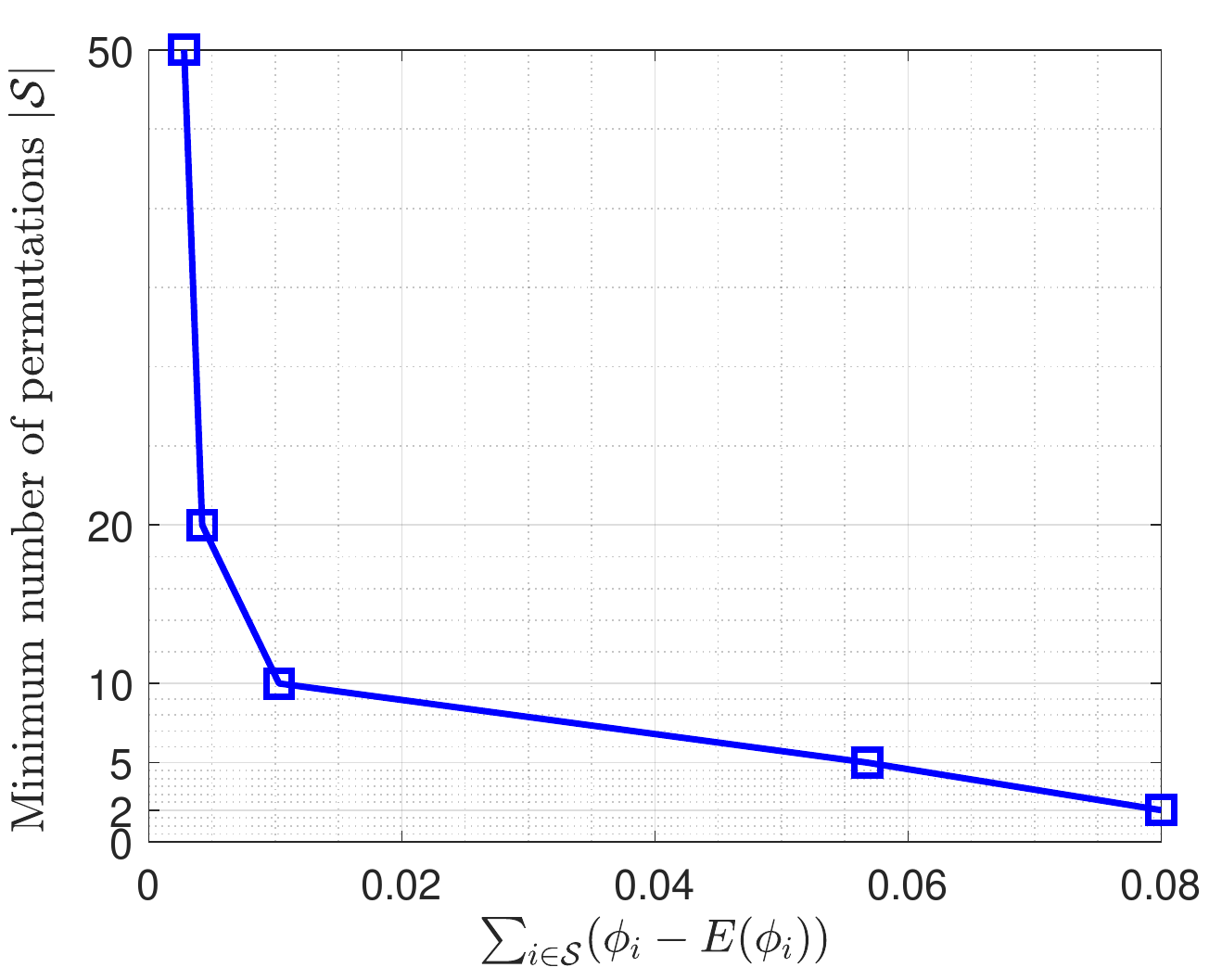}
    \caption{Performance analysis of AFS algorithm. }
    \label{fig:sampling_bound}
\end{figure}
\srpnew{In Fig.~\ref{fig:sampling_bound}, we show the results on performance analysis of AFS algorithm. We observed variability in the minimum permutation $|\mathcal{S}|$ required to ensure a defined deviation between the average value of contributions across clients, as measured using AFS, and the standard SV. For tighter bounds, the number of required permutations is large. This is intuitive, as the distributed ledger expects a larger sampling value $|\mathcal{S}|$ to define better confidence bound on the performance that minimizes the approximation error using AFS.}

\section{Performance Evaluation}

\subsection{Experimental Settings}
\lam{To demonstrate the applicability of our proposed system, we implement a proof-of-concept for the trading model in an IoT network. In this section, we introduce enabling technologies involved with the prototype.}

\subsubsection{Distributed Ledger} In this study, we implement Ethereum platform for the experimental.  Ethereum\footnote{https://ethereum.org/} is a distributed public blockchain network that focuses on running programming code of any decentralized application. Specifically, Ethereum is a platform for sharing information across the globe that cannot be manipulated or changed. Ethereum has its own cryptocurrency, called \textit{Ether} (ETH), and its own programming language, called \textit{Solidity}. The decentralized applications on the network is called \textit{Đapps}. Practically, Ethereum provides a convenient platform for development and smart contracts system to integrate with FL. We run Ethereum network via Ganache\footnote{https://www.trufflesuite.com/ganache} which is a personal blockchain for rapid Ethereum distributed application development. 
 
\subsubsection{Datasets} In the scope of this study,  we conducted the experiments on the MNIST data set. The dataset contains around 60,000 training images and over 10,000 testing images. Each client holds a part of dataset locally depending on the scenarios. 

\subsubsection{IoT Devices and Workstation} We use Raspberry Pi 3 with the following \srp{configurations}: Pytorch, OS Raspbian GNU/Linux 10, and Python version 3.7. We note that CUDA is not available for the model. The workstation has the system configurations as CPU i7-7700HQ, GPU GTX, Pytorch, OS Linux Ubuntu 20.04 , Python version 3.7.8 using Anaconda, and the CUDA version 11. \lam{These IoT devices are connected via Wifi access point.} 
%\bs{how are the IoT devices connected? see my previous comment about comms. models}

\begin{figure}[t!]
    \centering
    \includegraphics[width=0.8\linewidth]{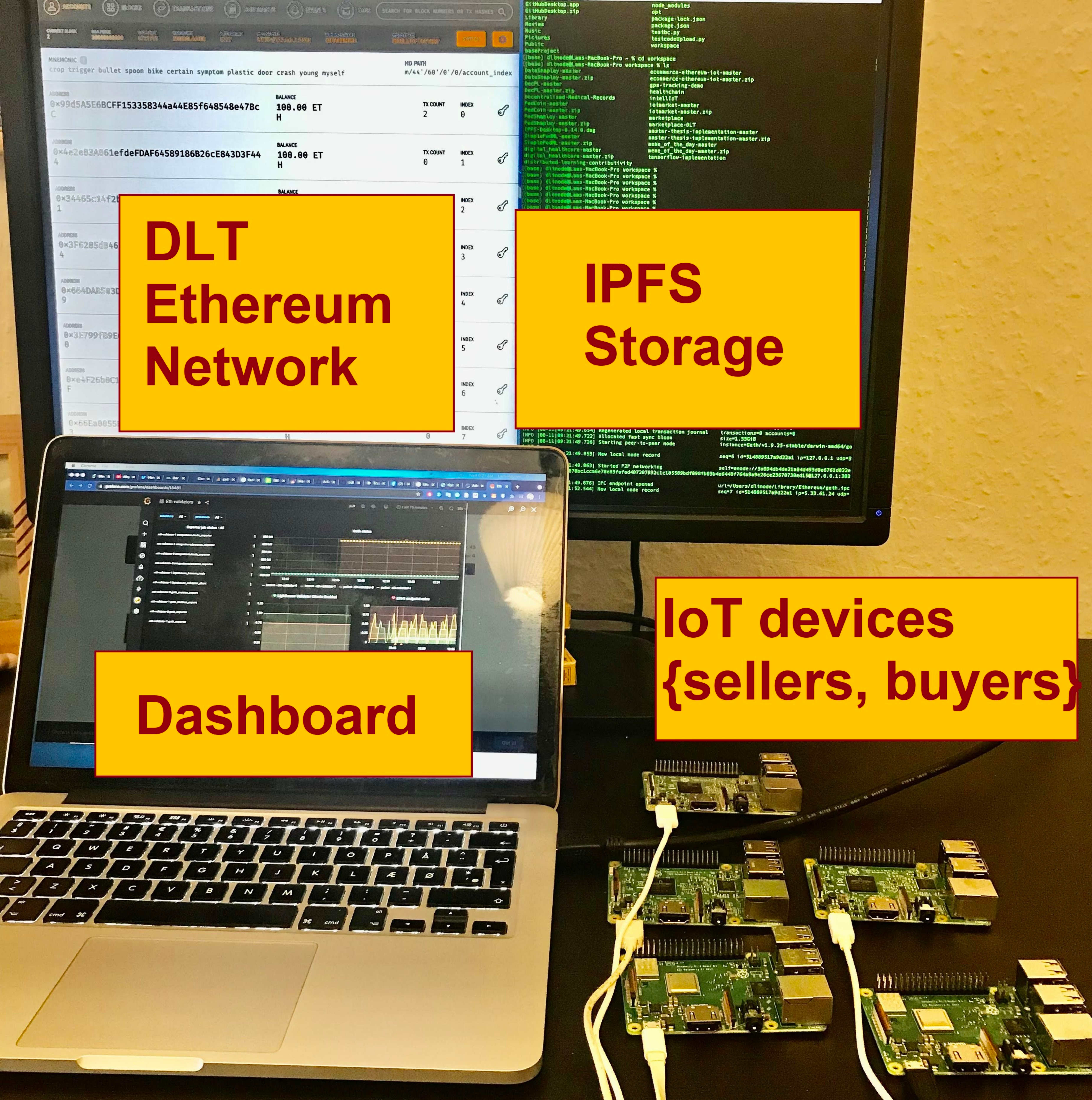}
    \caption{Blockchain-enabled model trading testbed. The testbed includes DLT Ethereum network running over Ganache, IPFS storage to address scalability issue, monitor dashboard and 5 IoT raspberry devices standing for marketplace participants as well as DLT clients.}
    \label{fig:my_label}
\end{figure}

\subsubsection{Evaluation Metrics.} We consider several performance metrics for comparison. 
\begin{itemize}
    \item \textbf{Cost of Smart Contract:} We study the fees made by users to compensate for the computing energy required to process and validate transactions on the Ethereum. 
    \item \textbf{Incentive per worker:} The amount of tokens delivered to sellers based on their contributions of training the model. 
    \item \textbf{Maximum Different}: The performance score function $U$ is chosen to be the accuracy function. The SVs are then calculated according to the different schemes. For comparison of the accuracy of the SV, all SVs calculated are first standardized by scaling them by a common factor such that $\sum_{i=1}^n \phi_i=1$.
\end{itemize}

This is appropriate because profit distribution will likely be based on the percentage contribution. Then, the \textbf{maximum different} $D_{max}$ measures the maximum difference that a data provider should be allocated by the definition and by approximated calculation. The calculation is shown as below: 
   
\begin{equation}
    \srp{D_{\max} = \max_{i\in\{1, \ldots, n\}} | \phi_i^F - \phi_i^S|}.  
\end{equation}

\subsubsection{Scenarios}
\begin{itemize}
    \item \textbf{(S1)}. In scenario 1, the compared algorithms have same distribution with same dataset size, i.e., each client dataset $D_i, \forall i \in \mathcal{N}$ has the same amount of training image samples. 
    \item  \textbf{(S2)}. In scenario 2, we introduce the case with same distribution but different dataset size. The training set is divided randomly into 5 parts with the same ratio of data size.
    
    \item \textbf{(S3)}. In scenario 3, we use different distribution with same dataset size. Each client's dataset $D_i, \forall i \in \{1,2,3,4,5\}$ has the same size, but the training images are not equally divided for each digit.
    %$40\%$ of $D_i$ is of digit $(2i-2)$, another $40\%$ is of digit $(2i-1)$ and the remaining 8 digits equally share the remaining $20\%$.
    
    \item \textbf{(S4)}. In scenario 4, we consider the case having an added noise feature with same dataset. First, we split the training set in a similar manner as \textbf{(S1)}. Afterwards, we generate Gaussian noise for the dataset. This is done by adjusting the standard deviation of the normal distribution. 
\end{itemize}

\begin{table}[t!]
\centering
\caption{EXECUTION COSTS OF SMART CONTRACTS}
 \begin{tabular}{@{}  p{2.5cm}  p{1.2cm}  p{1.2cm} p{1.2cm} p{1.2cm}  @{}}
 \toprule
\textbf{Smart Contracts} &\hspace{0.2cm} \textbf{From} & \hspace{0.1cm} \textbf{Gas} & \hspace{0.1cm} \textbf{Ether} & \hspace{0.1cm} \textbf{USD} \\ [0.1ex] 
 \midrule
Contract Registry      & 0x283D382F & 1459430 & 15.9$\cdot 10^{-5}$ & 0.0723  \\ 
AddWorker              & 0x283D382F & 452467 & 45.2$\cdot 10^{-5}$ & 0.0692 \\
AddWorker              & 0x283D382F & 452545 & 45.2$\cdot 10^{-5}$ & 0.0692 \\
AddWorker              & 0x283D382F & 452436 & 45.2$\cdot 10^{-5}$ & 0.0692 \\
AddWorker              & 0x283D382F & 452545 & 45.2$\cdot 10^{-5}$ & 0.0692 \\
AddWorker              & 0x283D382F & 452436 & 45.2$\cdot 10^{-5}$ & 0.0692 \\
ModelTransmission      & 0x5846F427 & 19374 & 19.3$\cdot 10^{-5}$& 0.1621 \\
ModelTransmission      & 0x9dD8Fd06 & 243482 & 24.3$\cdot 10^{-5}$& 0.0902 \\ 
ModelTransmission      & 0x98HF8F94 & 228779 & 22.3$\cdot 10^{-5}$& 0.1121 \\
ModelTransmission      & 0x8H9FH780 & 253924 & 25.3$\cdot 10^{-5}$& 0.0951 \\ 
ModelTransmission      & 0x0932FD99 & 263924 & 19.3$\cdot 10^{-5}$& 0.0571 \\ 
ModelTraining          & 0x5846F427 & 223924 & 22.3$\cdot 10^{-5}$& 0.1021 \\ 
ModelTraining          & 0x9DD8Fd06 & 253924 & 25.3$\cdot 10^{-5}$& 0.0951 \\ 
ModelTraining          & 0x98HF8F94 & 193924 & 19.3$\cdot 10^{-5}$& 0.0571 \\
ModelTraining          & 0x8H9FH780 & 253924 & 25.3$\cdot 10^{-5}$& 0.0951 \\ 
ModelTraining          & 0x0932FD99 & 253924 & 19.3$\cdot 10^{-5}$& 0.0571 \\ 
ModelAggregation       & 0x5846F427 & 324942 & 32.4$\cdot 10^{-5}$& 0.0766 \\
ModelAggregation       & 0x9dD8Fd06 & 283445 & 22.4$\cdot 10^{-5}$& 0.0408 \\ 
ModelAggregation       & 0x98HF8F94 & 214939 & 21.4$\cdot 10^{-5}$& 0.0709 \\
ModelAggregation       & 0x8H9FH780 & 253924 & 25.3$\cdot 10^{-5}$& 0.0951 \\ 
ModelAggregation       & 0x0932FD99 & 193924 & 19.3$\cdot 10^{-5}$& 0.0571 \\ 
Settlement             & 0x283D382F & 212559 & 21.3$\cdot 10^{-5}$& 0.0712 \\ 
PayChannelExecute      & 0x283D382F & 212538 & 21.2$\cdot 10^{-5}$& 0.0702 \\ 
\bottomrule
\end{tabular}
    \begin{tablenotes}
   % \centering
      \small
      \item * 1 Ether = $10^9$ Gwei; 1 USD = 246,940.5627 Gwei
    \end{tablenotes}
\label{tab:gas}
\end{table}

\subsection{Results}

\begin{figure}[t!]
    \centering
    \begin{subfigure}{0.48\textwidth}
    \includegraphics[width=0.95\linewidth]{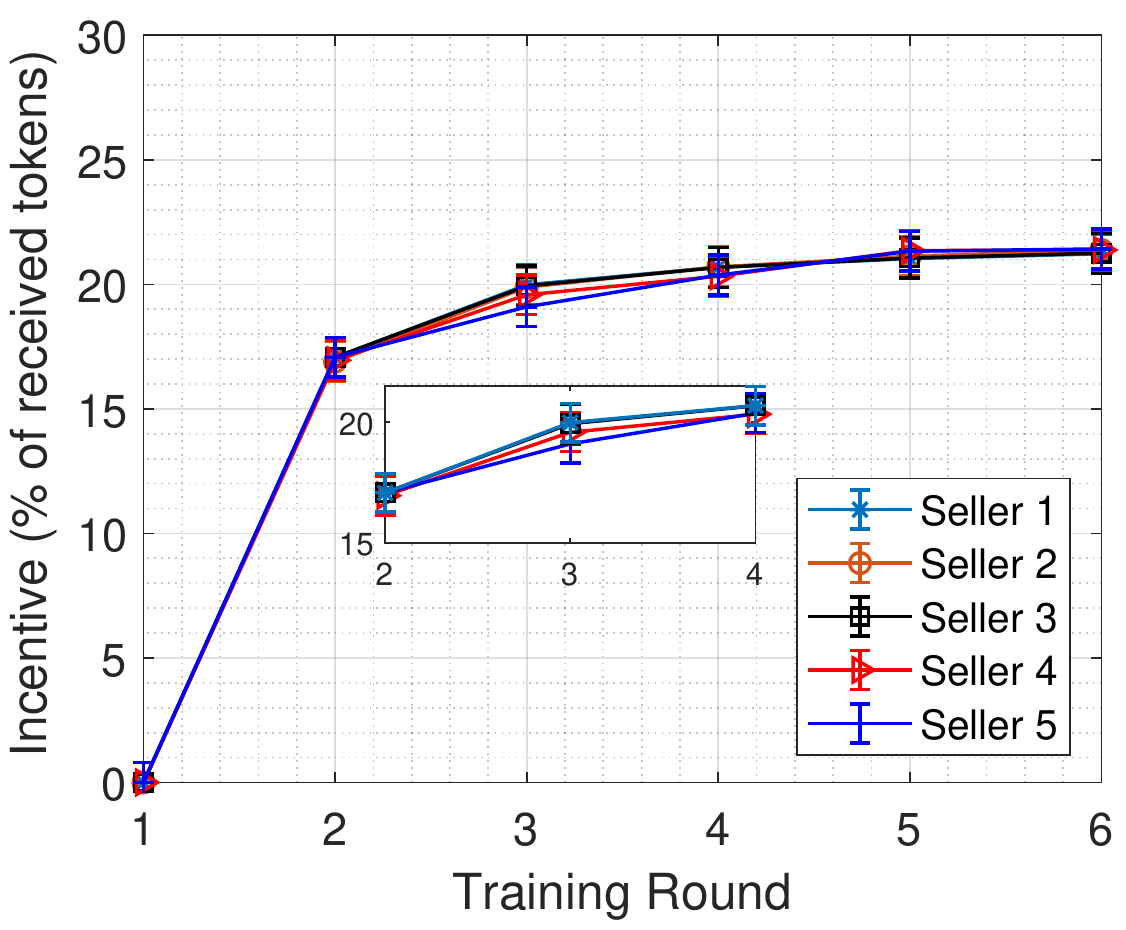} 
    \caption{Received incentive per client with the dataset distribution ratio 2:2:2:2:2.}
    \label{fig:incentive-a}%
    \end{subfigure}
    \begin{subfigure}{0.48\textwidth}
    \includegraphics[width=0.95\linewidth]{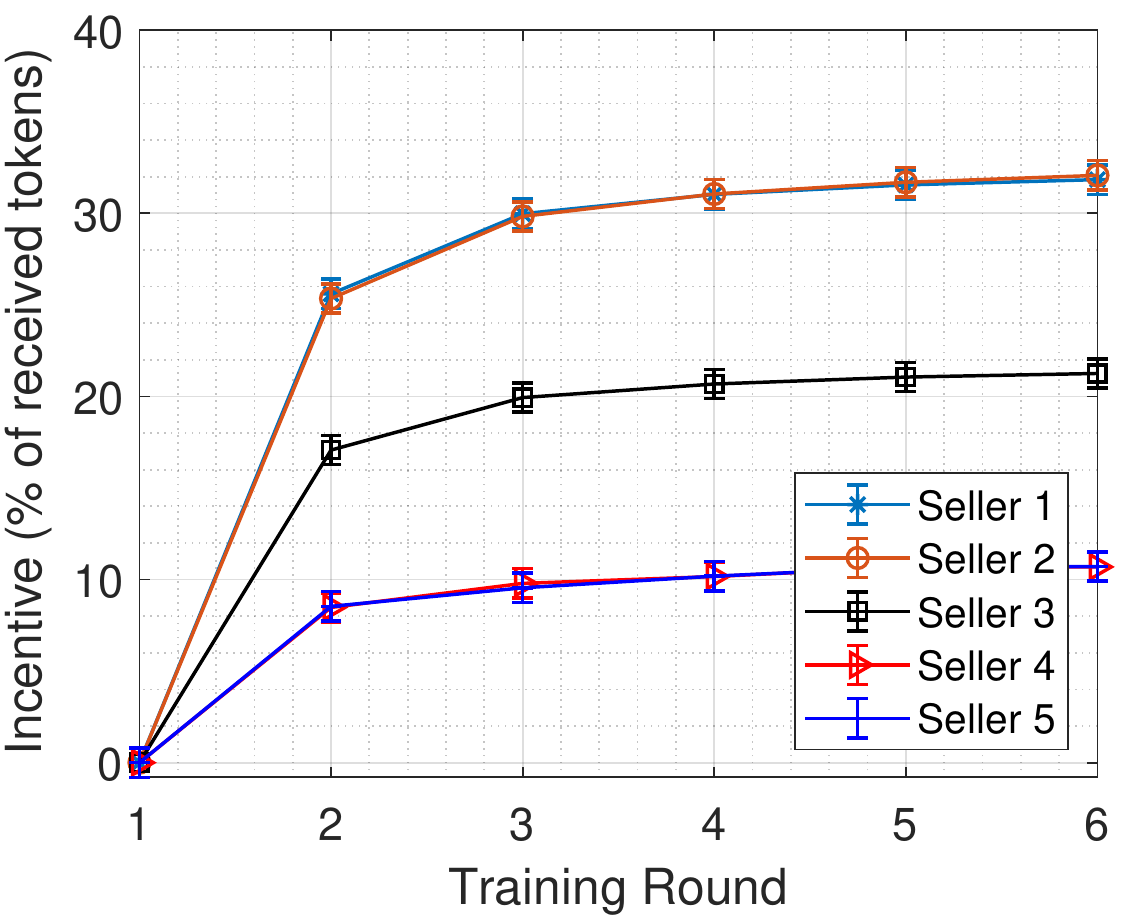}
    \caption{Received incentive per client with the dataset distribution ratio 3:3:2:1:1.}
    \label{fig:incentive-b}
    \end{subfigure}%
\caption{Incentive of clients.}
\label{fig:incentive}
\end{figure}

\subsubsection{Smart Contract Execution Cost}
\rv{In this part, the proof-of-concept of proposed model trading platform is deployed in a private Ethereum Blockchain called \textit{Ganache}\footnote{https://www.trufflesuite.com/ganache}. In distributed application \textit{Dapps}, the smart contract plays as key role in controlling and autonomously executing pre-defined agreements between the participants. We implemented and tested smart contracts using Remix IDE\footnote{https://remix.ethereum.org/}. In Ethereum network, there is a fee called \emph{gas}, needed to pay for any operation or transaction execution that changes the DLT states, which guarantees that smart contracts running in Ethereum Virtual Machine (EVM)\cite{ethereum} will be terminated eventually. In the scope of this research, we used Gwei\footnote{https://www.cryps.info/} to evaluate the cost of different operations, for example, \emph{AddWorker, ModelTransmisson, ModelTraning,} or \emph{Settlement} in the model trading process. The result is demonstrated in Table. \ref{tab:gas}.}

\begin{figure}[t]
    \centering
    \includegraphics[width=0.98\linewidth]{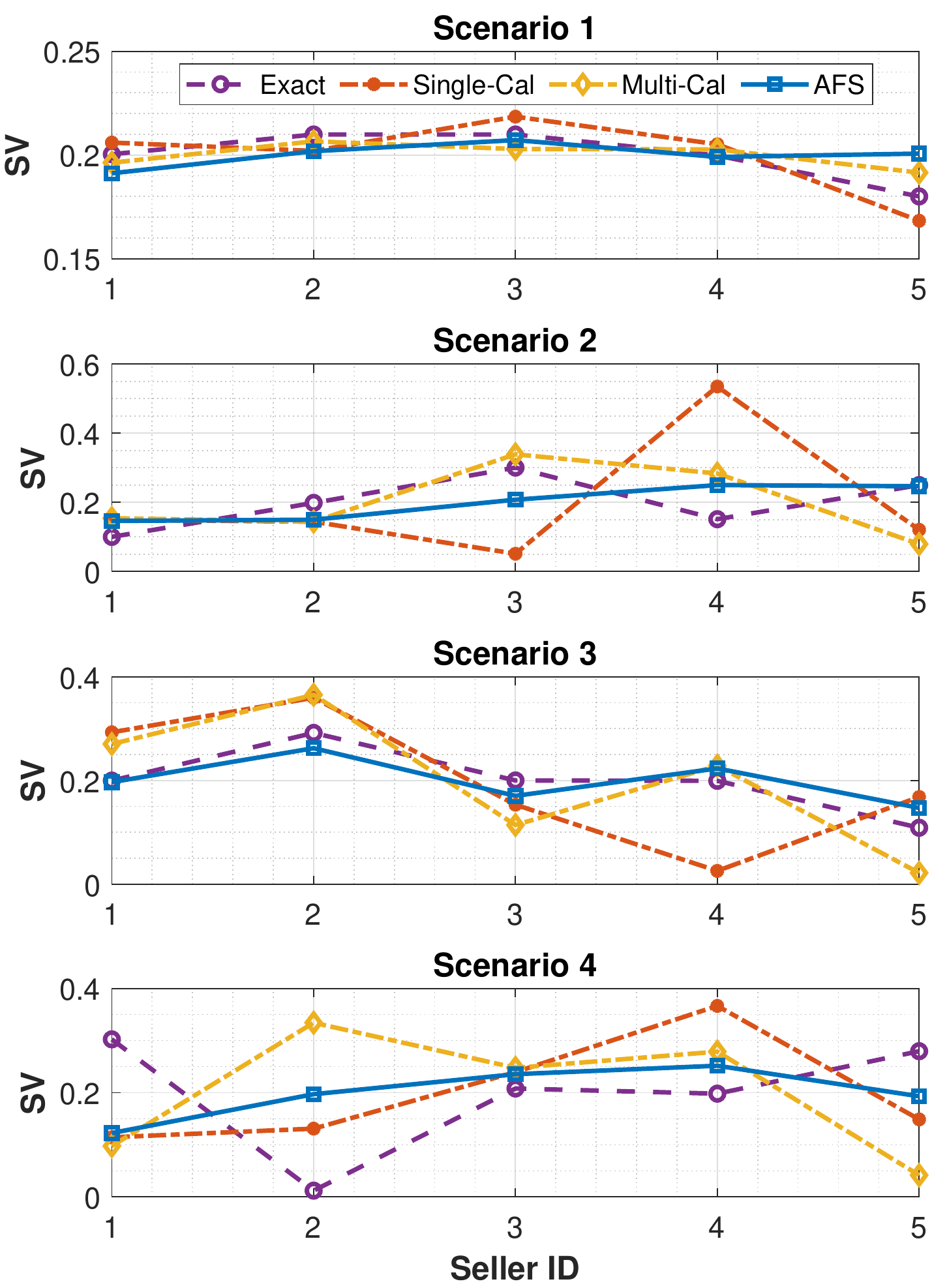}
    \caption{Shapley Value of each seller in different methods}
    \label{fig:shapleyValue}
\end{figure}

\begin{figure}[t!]
    \centering
    \includegraphics[width=0.98\linewidth]{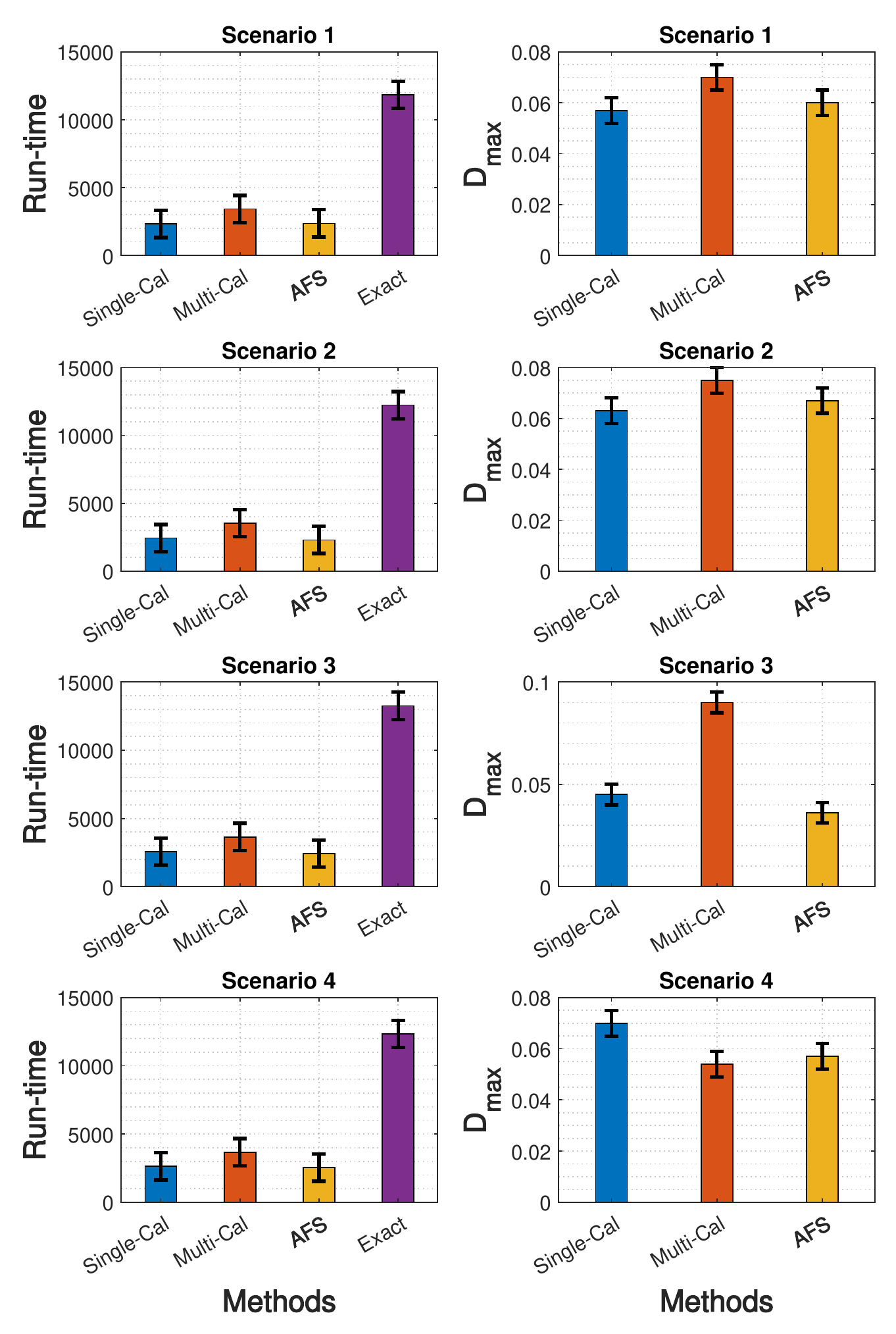}
    \caption{Comparison of execution time and $D_{\max}$ between algorithms.}
    \label{fig:runtim}
\end{figure}

\subsubsection{Incentive per client}
In Fig.~\ref{fig:incentive}, we show the comparison of received incentives by each training client based on their contribution to the global model training. The incentive is equivalent to tokens clients receive. As expected, in Fig. \ref{fig:incentive-a}, where the MNIST dataset is divided equally with a ratio of 2:2:2:2:2 for five involved clients, the amount of tokens they receive are almost similar as expected. In Fig. \ref{fig:incentive-b}, we show the comparison of the received percentage of tokens that clients can achieve with the dataset ratio of 3:3:2:1:1. We observe that client 1 and client 2 has the same amount of dataset, so they receive the same amount of tokens for their contribution, similar to the case for clients 4 and 5. \lam{Note that the sellers can train the models with poor quality, which, in fact, reduces the stability and performance of the global models. In this regard, there exist several mechanisms to handle such dishonest reporting of parameters in the FL setting, such as \cite{Xiong, niu2018achieving, agarwal2019marketplace}. Similar to this, the DLT keeps track of the contribution of devices and the gradient information and the size of data samples to regularly infer (check) the relationship between the expected model quality, reported data samples, and the obtained SV as Fig. \ref{fig:shapleyValue}; hence, dealing untruthful reporting. However, the detailed study of this mechanism is out of scope for this work. In Fig. 8, the \emph{AFS} shows a better performance while other methods turns out quite random SVs, especially in scenario 2 and 4 where the size of dataset is random and noise added.} 

\subsubsection{Execution time and maximum different comparison}
In Fig. \ref{fig:runtim}, we show the time performance of exact FL, Single-Cal, Multi-Cal, and AFS protocol. The Multi-Cal algorithm is more computational expensive than the Single-Cal algorithm. The standard exact method is the slowest one because the standard Shapley Value is naturally not compatible with the Federated Learning. In the scenario 1, each worker has same quality and quantity of dataset, so that we expect each worker has same contribution and receive equally the amount of incentive. The results show that the Single-Cal and AFS algorithm have higher efficiency in execution time. The exact method is around 5 times slower than the rest of methods because of frequent model retrain process. Meanwhile, the $D_{max}$ of methods are relatively low, around $0.05$. Similar in scenario 2 with the same size of dataset and different distribution, AFS and Single-Cal have better performance in running time and the accuracy. However, in scenario 4 with more noisy data, the Multi-Cal shows better results, $\approx 10\%$ in run-time but and $\approx 15\%$ in terms of maximum different value.

\section{Related Works}

%\bs{why is this section at the end?}

In this section, we first present the current works on asset trading based on Blockchain and data valuation. 

\textbf{Blockchain-based asset trading}. With the spread of ubiquitous marketplaces, \PP{it became relevant to explore} the application of IoT data trading in marketplace environments. \rv{For instance, the authors in \cite{IoTMartetplace} considered a dynamic decentralized marketplace and introduced the architecture for trading IoT data accordingly. }The approach involves a 3-tier method is used: 1) data provider, 2) broker and 3) data consumer. The primary purpose of DLTs in their function is to manage the conditions of agreements between the parties involved. In addition, the design has a reputation system that penalizes members and lowers their rating. The authors in \cite{iyengar2019towards} invested the optimization problem of revenue maximization with envy-free guarantee. The authors studied two scenarios including unit demand consumers and single minded consumers, and showed the optimization problem is APX-hard for both scenarios, which can be efficiently addressed by a logarithmic approximation. \rv{The authors in \cite{IoTMartetplace2} took into account the trading of IoT streaming data with the presented marketplace model, where fraudulent activity during data exchange is limited. To do so, the authors introduced periodic checkpoints during data trading.} In \cite{Missier}, the authors proposed another marketplace which flows of IoT data are the main digital assets exchanged utilizing Oracles for the off-chain queries.The authors in \cite{Xiong} presented a \PP{trading mode based on smart contracts}. \rv{In particular, the authors employ arbitration that handles disputes during the data trading, particularly, over the data availability, and incorporates AI/ML to ensure fairness during data exchange.}
 
\textbf{Data Valuation}. Evaluating the value of data has been received significant attention from both academia and industrial areas. 
\rv{Several works studied data valuation strategies and their applications. In this regard, the authors in \cite{wang2020principled} defined the data valuation in several categories, such as: (i) query-based pricing, where prices are attached to user-initiated queries \cite{kairouz2019advances}, \cite{leroy2019federated}, \cite{upadhyaya2016price}, (ii) data attribute-based pricing, where the price model considers data attributes, such as the age of data and its credibility, using the mechanism of public price registries \cite{heckman2015pricing}, and (iii) auction-based pricing, where the price is dynamically set following auction mechanisms \cite{mihailescu2010dynamic,lee2010sell }.}
In \cite{jia2019towards},\rv{multiple approximation strategies for optimizing the computation complex of SV for training data are introduced. Besides, the authors proposed an soltiion to compute exact SC in specific scenario, e.g nearest neighbor classifiers. Besides, the SV also is applied in various AI/ML application, for example, to measure the importance of model features\cite{cohen2007feature, strumbelj2010efficient}. In specific, the authors addressed the problem when the same data points get the same values, and relationship between data distributions and SV function. In addition, the authorsproposed an idea of distributional SV occurs resemblance to the Aumann-SV\cite{aumann2015values}. In practical manner, the authors in \cite{tang2021data} proved that the performance of model training can be improved by removing the data with low SV value. In contrast, the performance will be decreased if we deleting the training data with high SV values. }

\section{Conclusion}
In this paper, we proposed a DLT-based marketplace for trading ML models, which helps companies and organizations train their learning models \arne{in a scalable and efficient manner}. An incentive mechanism exists to stimulate participants in joining and training the learning models on the marketplace, which pays participants based on their contributions to train the model. To that end, an extended Data Shapley Value (DSV) for the federated environment is proposed to measure each participant's contribution in the model training process. Finally, with extensive experimental evaluations with Ethereum Blockchain to build a marketplace for model trading using smart contracts and IoT devices acting as participants, we demonstrated the design and performance of the proposed ecosystem. 

\section{Acknowledgment}

This work has received funding from the European Union’s Horizon 2020 research and innovation programme under grant agreement No. 957218 (Project IntellIoT).

\bibliographystyle{ieeetr}
\bibliography{MAIN}

\begin{thebibliography}{10}

\bibitem{data794}
I.~Report, ``The growth in connected iot devices is expected to generate 79.4zb
  of data in 2025, according to a new idc forecast.''
  [Online]\url{https://www.idc.com}, 2019.
\newblock (Accessed on 12/04/2020).

\bibitem{previous1}
Z.~Huang, X.~Su, Y.~Zhang, C.~Shi, H.~Zhang, and L.~Xie, ``A decentralized
  solution for iot data trusted exchange based-on blockchain,'' in {\em 2017
  3rd IEEE International Conference on Computer and Communications (ICCC)},
  pp.~1180--1184, IEEE, 2017.

\bibitem{previous2}
C.~Perera, ``Sensing as a service (s2aas): Buying and selling iot data,'' {\em
  arXiv preprint arXiv:1702.02380}, 2017.

\bibitem{infocom2019}
W.~Mao, Z.~Zheng, and F.~Wu, ``Pricing for revenue maximization in iot data
  markets: An information design perspective,'' in {\em IEEE INFOCOM 2019-IEEE
  Conference on Computer Communications}, pp.~1837--1845, IEEE, 2019.

\bibitem{bishoi2009comparative}
B.~Bishoi, A.~Prakash, V.~Jain, {\em et~al.}, ``A comparative study of air
  quality index based on factor analysis and us-epa methods for an urban
  environment,'' {\em Aerosol and Air Quality Research}, vol.~9, no.~1,
  pp.~1--17, 2009.

\bibitem{jo2020development}
J.~Jo, B.~Jo, J.~Kim, S.~Kim, and W.~Han, ``Development of an iot-based indoor
  air quality monitoring platform,'' {\em Journal of Sensors}, vol.~2020, 2020.

\bibitem{9184079}
Y.~Liu, J.~Nie, X.~Li, S.~H. Ahmed, W.~Y.~B. Lim, and C.~Miao, ``Federated
  learning in the sky: Aerial-ground air quality sensing framework with uav
  swarms,'' {\em IEEE Internet of Things Journal}, vol.~8, no.~12,
  pp.~9827--9837, 2021.

\bibitem{moltchanov2015feasibility}
S.~Moltchanov, I.~Levy, Y.~Etzion, U.~Lerner, D.~M. Broday, and B.~Fishbain,
  ``On the feasibility of measuring urban air pollution by wireless distributed
  sensor networks,'' {\em Science of The Total Environment}, vol.~502,
  pp.~537--547, 2015.

\bibitem{gruschka2018privacy}
N.~Gruschka, V.~Mavroeidis, K.~Vishi, and M.~Jensen, ``Privacy issues and data
  protection in big data: a case study analysis under gdpr,'' in {\em 2018 IEEE
  International Conference on Big Data (Big Data)}, pp.~5027--5033, IEEE, 2018.

\bibitem{mcmahan2017communication}
B.~McMahan, E.~Moore, D.~Ramage, S.~Hampson, and B.~A. y~Arcas,
  ``Communication-efficient learning of deep networks from decentralized
  data,'' in {\em Artificial intelligence and statistics}, pp.~1273--1282,
  PMLR, 2017.

\bibitem{xu2014ubiquitous}
B.~Xu, L.~Da~Xu, H.~Cai, C.~Xie, J.~Hu, and F.~Bu, ``Ubiquitous data accessing
  method in iot-based information system for emergency medical services,'' {\em
  IEEE Transactions on Industrial informatics}, vol.~10, no.~2, pp.~1578--1586,
  2014.

\bibitem{radhakrishnan2018streaming}
R.~Radhakrishnan and B.~Krishnamachari, ``Streaming data payment protocol
  (sdpp) for the internet of things,'' in {\em 2018 IEEE International
  Conference on Internet of Things (iThings) and IEEE Green Computing and
  Communications (GreenCom) and IEEE Cyber, Physical and Social Computing
  (CPSCom) and IEEE Smart Data (SmartData)}, pp.~1679--1684, IEEE, 2018.

\bibitem{niu2018achieving}
C.~Niu, Z.~Zheng, F.~Wu, X.~Gao, and G.~Chen, ``Achieving data truthfulness and
  privacy preservation in data markets,'' {\em IEEE Transactions on Knowledge
  and Data Engineering}, vol.~31, no.~1, pp.~105--119, 2018.

\bibitem{langevoort1992fraud}
D.~C. Langevoort, ``Fraud and insider trading in american securities
  regulation: Its scope and philosophy in a global marketplace,'' {\em Hastings
  Int'l \& Comp. L. Rev.}, vol.~16, p.~175, 1992.

\bibitem{pandey2020crowdsourcing}
S.~R. Pandey, N.~H. Tran, M.~Bennis, Y.~K. Tun, A.~Manzoor, and C.~S. Hong, ``A
  crowdsourcing framework for on-device federated learning,'' {\em IEEE
  Transactions on Wireless Communications}, vol.~19, no.~5, pp.~3241--3256,
  2020.

\bibitem{bitcoin}
S.~Nakamoto, ``Bitcoin: A peer-to-peer electronic cash system,'' tech. rep.,
  Manubot, 2008.

\bibitem{nguyen2021modeling}
L.~D. Nguyen, I.~Leyva-Mayorga, A.~N. Lewis, and P.~Popovski, ``Modeling and
  analysis of data trading on blockchain-based market in iot networks,'' {\em
  IEEE Internet of Things Journal}, vol.~8, no.~8, pp.~6487--6497, 2021.

\bibitem{9106844}
B.~Chen, D.~He, N.~Kumar, H.~Wang, and K.-K.~R. Choo, ``A blockchain-based
  proxy re-encryption with equality test for vehicular communication systems,''
  {\em IEEE Transactions on Network Science and Engineering}, vol.~8, no.~3,
  pp.~2048--2059, 2021.

\bibitem{iotmagazine}
L.~D. {Nguyen}, A.~E. {Kalor}, I.~{Leyva-Mayorga}, and P.~{Popovski}, ``Trusted
  wireless monitoring based on distributed ledgers over nb-iot connectivity,''
  {\em IEEE Communications Magazine}, vol.~58, no.~6, pp.~77--83, 2020.

\bibitem{9426434}
T.~Wang, C.~Zhao, Q.~Yang, S.~Zhang, and S.~C. Liew, ``Ethna: Analyzing the
  underlying peer-to-peer network of ethereum blockchain,'' {\em IEEE
  Transactions on Network Science and Engineering}, vol.~8, no.~3,
  pp.~2131--2146, 2021.

\bibitem{kairouz2019advances}
P.~Kairouz, H.~B. McMahan, B.~Avent, A.~Bellet, M.~Bennis, A.~N. Bhagoji,
  K.~Bonawitz, Z.~Charles, G.~Cormode, R.~Cummings, {\em et~al.}, ``Advances
  and open problems in federated learning,'' {\em arXiv preprint
  arXiv:1912.04977}, 2019.

\bibitem{kim2019blockchained}
H.~Kim, J.~Park, M.~Bennis, and S.-L. Kim, ``Blockchained on-device federated
  learning,'' {\em IEEE Communications Letters}, vol.~24, no.~6,
  pp.~1279--1283, 2019.

\bibitem{roth1988shapley}
A.~E. Roth, {\em The Shapley value: essays in honor of Lloyd S. Shapley}.
\newblock Cambridge University Press, 1988.

\bibitem{ethereum}
G.~Wood {\em et~al.}, ``Ethereum: A secure decentralised generalised
  transaction ledger,'' {\em Ethereum project yellow paper}, vol.~151,
  no.~2014, pp.~1--32, 2014.

\bibitem{jia2019towards}
R.~Jia, D.~Dao, B.~Wang, F.~A. Hubis, N.~Hynes, N.~M. G{\"u}rel, B.~Li,
  C.~Zhang, D.~Song, and C.~J. Spanos, ``Towards efficient data valuation based
  on the shapley value,'' in {\em The 22nd International Conference on
  Artificial Intelligence and Statistics}, pp.~1167--1176, PMLR, 2019.

\bibitem{9271890}
X.~Ding, J.~Guo, D.~Li, and W.~Wu, ``An incentive mechanism for building a
  secure blockchain-based internet of things,'' {\em IEEE Transactions on
  Network Science and Engineering}, vol.~8, no.~1, pp.~477--487, 2021.

\bibitem{qu2020decentralized}
Y.~Qu, L.~Gao, T.~H. Luan, Y.~Xiang, S.~Yu, B.~Li, and G.~Zheng,
  ``Decentralized privacy using blockchain-enabled federated learning in fog
  computing,'' {\em IEEE Internet of Things Journal}, vol.~7, no.~6,
  pp.~5171--5183, 2020.

\bibitem{ghorbani2019data}
A.~Ghorbani and J.~Zou, ``Data shapley: Equitable valuation of data for machine
  learning,'' in {\em International Conference on Machine Learning},
  pp.~2242--2251, PMLR, 2019.

\bibitem{charles2020outsized}
Z.~Charles and J.~Kone{\v{c}}n{\`y}, ``On the outsized importance of learning
  rates in local update methods,'' {\em arXiv preprint arXiv:2007.00878}, 2020.

\bibitem{hoeffding1951combinatorial}
W.~Hoeffding, ``A combinatorial central limit theorem,'' {\em The Annals of
  Mathematical Statistics}, pp.~558--566, 1951.

\bibitem{Xiong}
W.~Xiong and L.~Xiong, ``Smart contract based data trading mode using
  blockchain and machine learning,'' {\em IEEE Access}, vol.~7,
  pp.~102331--102344, 2019.

\bibitem{agarwal2019marketplace}
A.~Agarwal, M.~Dahleh, and T.~Sarkar, ``A marketplace for data: An algorithmic
  solution,'' in {\em Proceedings of the 2019 ACM Conference on Economics and
  Computation}, pp.~701--726, 2019.

\bibitem{IoTMartetplace}
P.~Gupta, S.~Kanhere, and R.~Jurdak, ``A decentralized iot data marketplace,''
  {\em arXiv preprint arXiv:1906.01799}, 2019.

\bibitem{iyengar2019towards}
R.~Iyengar, J.~P. Near, D.~Song, O.~Thakkar, A.~Thakurta, and L.~Wang,
  ``Towards practical differentially private convex optimization,'' in {\em
  2019 IEEE Symposium on Security and Privacy (SP)}, pp.~299--316, IEEE, 2019.

\bibitem{IoTMartetplace2}
S.~Bajoudah, C.~Dong, and P.~Missier, ``Toward a decentralized, trust-less
  marketplace for brokered iot data trading using blockchain,'' in {\em 2019
  IEEE International Conference on Blockchain (Blockchain)}, pp.~339--346,
  IEEE, 2019.

\bibitem{Missier}
P.~Missier, S.~Bajoudah, A.~Capossele, A.~Gaglione, and M.~Nati, ``Mind my
  value: a decentralized infrastructure for fair and trusted iot data
  trading,'' in {\em Proceedings of the Seventh International Conference on the
  Internet of Things}, pp.~1--8, 2017.

\bibitem{wang2020principled}
T.~Wang, J.~Rausch, C.~Zhang, R.~Jia, and D.~Song, ``A principled approach to
  data valuation for federated learning,'' in {\em Federated Learning},
  pp.~153--167, Springer, 2020.

\bibitem{leroy2019federated}
D.~Leroy, A.~Coucke, T.~Lavril, T.~Gisselbrecht, and J.~Dureau, ``Federated
  learning for keyword spotting,'' in {\em ICASSP 2019-2019 IEEE International
  Conference on Acoustics, Speech and Signal Processing (ICASSP)},
  pp.~6341--6345, IEEE, 2019.

\bibitem{upadhyaya2016price}
P.~Upadhyaya, M.~Balazinska, and D.~Suciu, ``Price-optimal querying with data
  apis,'' {\em Proceedings of the VLDB Endowment}, vol.~9, no.~14,
  pp.~1695--1706, 2016.

\bibitem{heckman2015pricing}
J.~R. Heckman, E.~L. Boehmer, E.~H. Peters, M.~Davaloo, and N.~G. Kurup, ``A
  pricing model for data markets,'' {\em iConference 2015 Proceedings}, 2015.

\bibitem{mihailescu2010dynamic}
M.~Mihailescu and Y.~M. Teo, ``Dynamic resource pricing on federated clouds,''
  in {\em 2010 10th IEEE/ACM International Conference on Cluster, Cloud and
  Grid Computing}, pp.~513--517, IEEE, 2010.

\bibitem{lee2010sell}
J.-S. Lee and B.~Hoh, ``Sell your experiences: a market mechanism based
  incentive for participatory sensing,'' in {\em 2010 IEEE International
  Conference on Pervasive Computing and Communications (PerCom)}, pp.~60--68,
  IEEE, 2010.

\bibitem{cohen2007feature}
S.~Cohen, G.~Dror, and E.~Ruppin, ``Feature selection via coalitional game
  theory,'' {\em Neural Computation}, vol.~19, no.~7, pp.~1939--1961, 2007.

\bibitem{strumbelj2010efficient}
E.~Strumbelj and I.~Kononenko, ``An efficient explanation of individual
  classifications using game theory,'' {\em The Journal of Machine Learning
  Research}, vol.~11, pp.~1--18, 2010.

\bibitem{aumann2015values}
R.~J. Aumann and L.~S. Shapley, {\em Values of non-atomic games}.
\newblock Princeton University Press, 2015.

\bibitem{tang2021data}
S.~Tang, A.~Ghorbani, R.~Yamashita, S.~Rehman, J.~A. Dunnmon, J.~Zou, and D.~L.
  Rubin, ``Data valuation for medical imaging using shapley value and
  application to a large-scale chest x-ray dataset,'' {\em Scientific reports},
  vol.~11, no.~1, pp.~1--9, 2021.

\end{thebibliography}

%%%%%%%%%%% Biography

\begin{IEEEbiography}[{\includegraphics[width=0.9in,height=1.55in,clip,keepaspectratio]{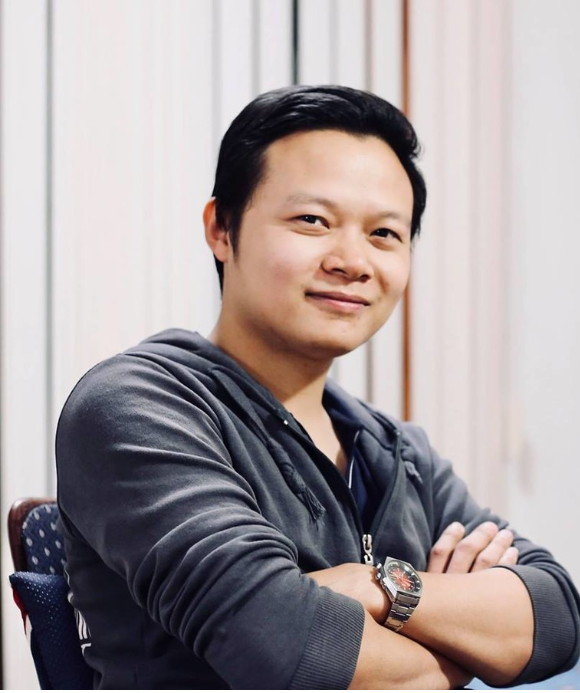}}] {Lam Duc Nguyen}
(S'20) is a Ph.D. Fellow at Aalborg University. He obtained Master Degree in Computer Science at Seoul National University, and a Bachelor in Telecommunication at Hanoi University of Science and Technologies in 2019 and 2015, respectively. His research includes Distributed Systems, Blockchain, Smart Contracts, the Internet of Things, and applying Blockchain and Federated Learning to enhance the efficiency of Blockchain-based IoT monitoring Networks. He receives Outstanding Paper Award for the research about scaling Blockchain in Massive IoT at the IEEE World Forum Internet of Things 2020, travel grant from Linux Foundation 2020, Best Research Award for a solution of Blockchain-based CO2 Emission Trading from VEHITS 2021. He is Hyperledger Member, IEEE Student Member, and IEEE ComSoc Student Member. 
\end{IEEEbiography}

\begin{IEEEbiography}[{\includegraphics[width=0.9in,height=1.55in,clip,keepaspectratio]{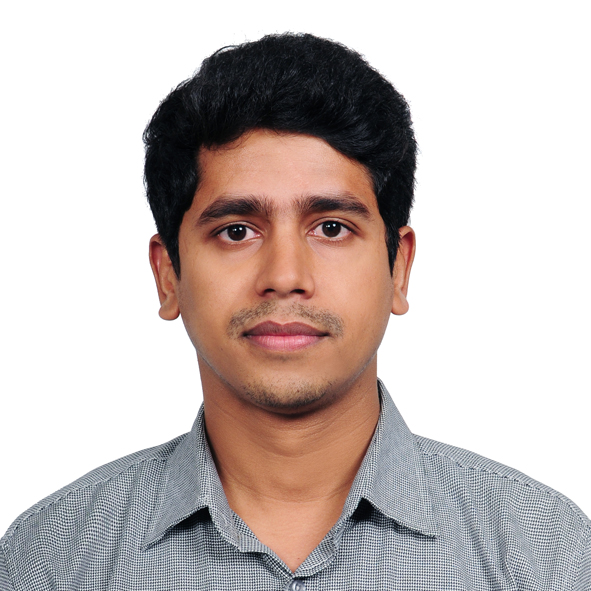}}] {Shashi Raj Pandey}(M'21) is currently working as a Postdoctoral Researcher at the Connectivity Section, Aalborg University. He received the B.E. degree in Electrical and Electronics with a specialization in Communication from Kathmandu University, Nepal in 2013, and the Ph.D. degree in Computer Science and Engineering from Kyung Hee University, Seoul, South Korea, in August, 2021. He served as a Network Engineer at Huawei Technologies Nepal Co. Pvt. Ltd, Nepal from 2013 to 2016. His research interests include network economics, game theory, wireless communications, data markets and distributed machine learning.
\end{IEEEbiography}

\begin{IEEEbiography}[{\includegraphics[width=0.9in,height=1.55in,clip,keepaspectratio]{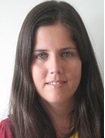}}] {Beatriz Soret}[M’11] received her M.Sc. and Ph.D. degrees in Telecommunications from the Universidad de Malaga, Spain, in 2002 and 2010, respectively. She is currently an associate professor at the Department of Electronic Systems, Aalborg University, and a Senior Research Fellow at the Communications Engineering Department, University of Malaga. Her research interests include LEO satellite communications, distributed and intelligent IoT, timing in communications, and 5G and post-5G systems.
\end{IEEEbiography}

\begin{IEEEbiography}[{\includegraphics[width=0.9in,height=1.55in,clip,keepaspectratio]{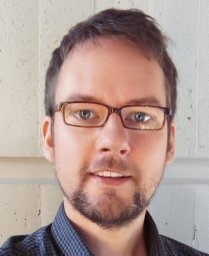}}] {Arne Br\"{o}ring}
is a Senior Key Expert Research Scientist at Siemens Technology in Munich. He received his PhD in 2012 from the University of Twente (Netherlands). Dr. Br\"{o}ring has contributed to over 90 publications in the field of distributed systems and has served on various program committees and editorial boards. His research interests range from distributed system designs, over sensor networks, and Semantic Web, to the Internet of Things. At Siemens, he has been in charge of the technical \& scientific coordination of large EU research projects (BIG IoT and IntellIoT). Before joining Siemens, Dr. Br\"{o}ring worked for the Environmental Systems Research Institute in Zurich, the 52°North Open Source Initiative, and led the Sensor Web and Simulation Lab at the University of Münster.
\end{IEEEbiography}

\begin{IEEEbiography}[{\includegraphics[width=0.9in,height=1.55in,clip,keepaspectratio]{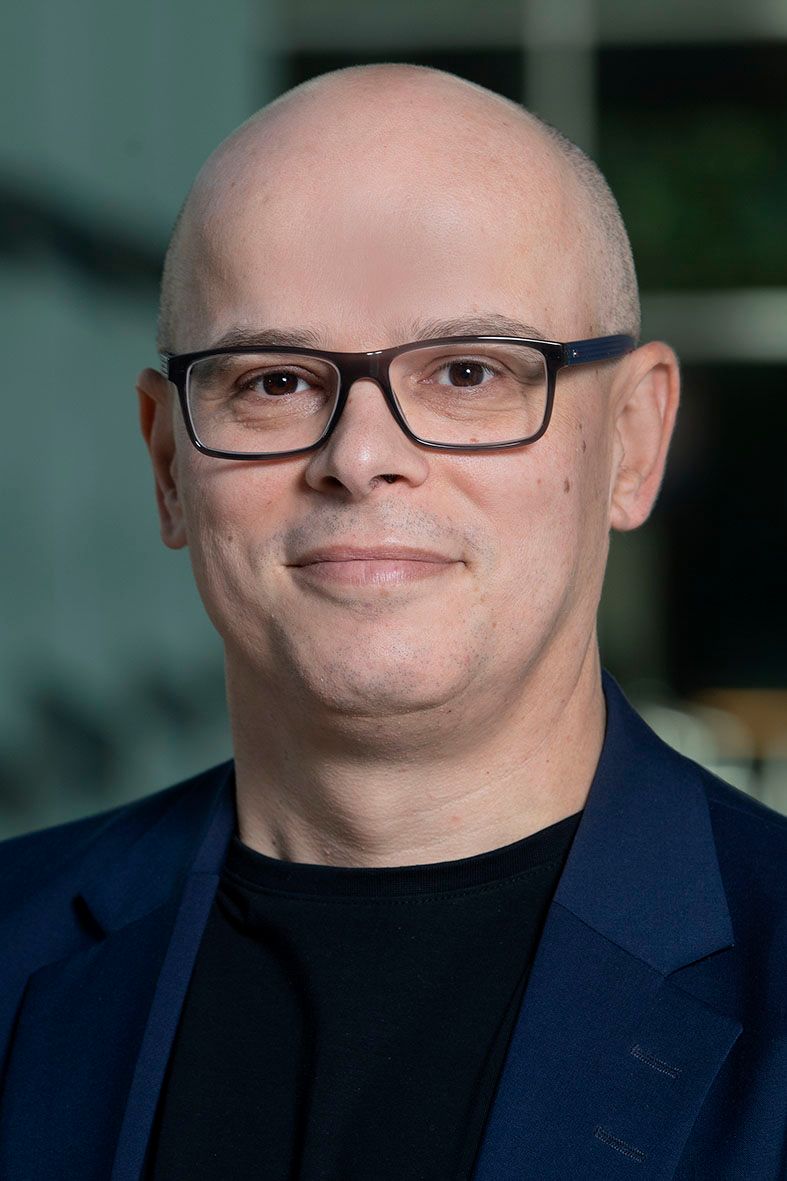}}] {Petar Popovski}
(Fellow, IEEE) is a Professor at Aalborg University, where he heads the section on Connectivity and a Visiting Excellence Chair at the University of Bremen. He received his Dipl.-Ing and M. Sc. degrees in communication engineering from the University of Sts. Cyril and Methodius in Skopje and the Ph.D. degree from Aalborg University in 2005. He is a Fellow of the IEEE. He received an ERC Consolidator Grant (2015), the Danish Elite Researcher award (2016), IEEE Fred W. Ellersick prize (2016), IEEE Stephen O. Rice prize (2018), Technical Achievement Award from the IEEE Technical Committee on Smart Grid Communications (2019), the Danish Telecommunication Prize (2020) and Villum Investigator Grant (2021). He is a Member at Large at the Board of Governors in IEEE Communication Society, Vice-Chair of the IEEE Communication Theory Technical Committee and IEEE TRANSACTIONS ON GREEN COMMUNICATIONS AND NETWORKING. He is currently an Area Editor of the IEEE TRANSACTIONS ON WIRELESS COMMUNICATIONS and, from 2022, an Editor-in-Chief of IEEEE JOURNAL ON SELECTED AREAS IN COMMUNICATIONS. Prof. Popovski was the General Chair for IEEE SmartGridComm 2018 and IEEE Communication Theory Workshop 2019. His research interests are in the area of wireless communication and communication theory. He authored the book ``Wireless Connectivity: An Intuitive and Fundamental Guide'', published by Wiley in 2020.
\end{IEEEbiography}

\end{document}